\documentclass{amsart}

\usepackage{graphicx}

\DeclareMathOperator{\grad}{grad}
\DeclareMathOperator{\diag}{diag}

\newtheorem{theorem}{Theorem}[section]
\newtheorem{lemma}[theorem]{Lemma}

\theoremstyle{definition}
\newtheorem{definition}[theorem]{Definition}

\newtheorem{corollary}[theorem]{Corollary}

\theoremstyle{remark}
\newtheorem{remark}[theorem]{Remark}

\numberwithin{equation}{section}
\usepackage{xcolor}
\usepackage{lineno}
\usepackage{subcaption}
\usepackage{comment}
\usepackage[normalem]{ulem}
\usepackage{amsthm}

\begin{document}


\title{Entropic regularization in the Deep Linear Network}


\author{Alan Chen}
\address{Division of Applied Mathematics, Brown University, Providence, RI 02912}
\curraddr{Unaffiliated}
\email{alan\_chen1@alumni.brown.edu}
\author{Tejas Kotwal}
\address{Division of Applied Mathematics, Brown University, Providence, RI 02912}
\email{tejas\_suresh\_kotwal@brown.edu}

\author{Govind Menon}
\address{Division of Applied Mathematics, Brown University, 182 George St., Providence, RI 02912.}
\email{govind\_menon@brown.edu}
\curraddr{School of Mathematics, Institute for Advanced Study, 1 Einstein Drive, Princeton, NJ 08540}
\email{gmenon@ias.edu}

\thanks{GM was supported by NSF grant DMS 2407055 and the Erik Ellentuck Fellow Fund at the Institute for Advanced Study, Princeton.}

\subjclass[2020]{37N40, 53C20, 15A18, 68T07}
\keywords{Deep linear network, Random matrix theory, Implicit regularization}



\begin{abstract}
We study regularization for the deep linear network (DLN) using the entropy formula introduced in~\cite{menon2025entropy}. The equilibria and gradient flow of the free energy $F_\beta = E - \beta^{-1} S_N$ on the Riemannian manifold $(\mathfrak{M}_d, g^N)$ of end--to--end maps of the DLN are characterized for energies $E(X)$ that depend symmetrically on the singular values of $X$. 

The only equilibria are minimizers and the set of minimizers is an orbit of the orthogonal group. In contrast with random matrix theory there is no singular value repulsion. The corresponding gradient flow reduces to a one-dimensional ordinary differential equation whose solution gives explicit relaxation rates toward the minimizers. We also study the concavity of the entropy $S_N(X)$ in the chamber of singular values. The entropy is shown to be strictly concave in the Euclidean geometry on the chamber but not in the Riemannian geometry defined by the metric $g^N$. 
\end{abstract}

\maketitle

{\centering \em For Percy Deift on the occasion of his 80th birthday.\par}

\section{Overview}
\subsection{Background} 
The deep linear network (DLN) is a phenomenological model for training dynamics in deep learning. It was introduced by Arora, Cohen and Hazan to analyze implicit regularization~\cite{arora2018optimization} and has given rise to a rich literature since (see~\cite{menon2024geometry} for an expository account of the underlying mathematics). The purpose of this paper is to relate the Boltzmann entropy introduced in~\cite{menon2025entropy} to the problem of regularization. Let us briefly explain the underlying context. 

Fix two positive integers $N$ and $d$ referred to as the depth and width of the network. Let $\mathbb{M}_d$ and $\mathfrak{M}_d$ denote the space of real $d\times d$ matrices and real invertible $d\times d$ matrices respectively and equip these spaces with the DLN metric $g^N$ defined in~\cite{bah2022learning,menon2024geometry} (we review this metric in Section~\ref{sec:geometry} below). The simplest form of implicit regularization in the DLN arises when we consider cost functions $E:\mathbb{M}_d \to \mathbb{R}$ that correspond to matrix sensing. Typically, such $E$ have an affine subspace of minimizers and numerical simulations show that for randomly chosen initial conditions the solution to the gradient flow 
\begin{equation}
\label{eq:gf1}
\dot{X} = -\mathrm{grad}_{g^N} E(X), \quad X \in \mathfrak{M}_d,    
\end{equation}
appears to converge to rank-deficient minimizers of $E$ ~\cite[\S 3.3.2]{cohen2023deep}.

The gradient flow~\eqref{eq:gf1} corresponds exactly to the training dynamics in the parameter space $\mathbb{M}_d^N$ with balanced initial conditions. Thus, the first step in the rigorous analysis of implicit regularization for matrix sensing is the analysis of long-time and transient dynamics of equation~\eqref{eq:gf1}. However, this system is subtle to analyze even when $d$ is as small as $2$. At present, we know that $\lim_{t \to \infty} X(t)$ exists for all initial conditions, but we lack methods that identify this limit.  
\subsection{Entropic regularization}
\label{subsec:results}
Our purpose in this work is to provide a rigorous selection criterion for cost functions that are regularized as follows. The Boltzmann entropy for the DLN with depth $N$ is defined by the formula~\cite[Theorem 4]{menon2025entropy}
\begin{equation} \label{eq:S_N_entropy}
    S_N(X)
    \;=\;
   (N-1)c_d+ \frac{1}{2}
    \sum_{1\le i<j\le d}
    \log\!\left(\frac{\sigma_i^2-\sigma_j^2}{\sigma_i^{2/N}-\sigma_j^{2/N}}\right),
\end{equation}
where $c_d$ is the volume of the orthogonal group $O_d$. Given an inverse temperature $\beta>0$, we use the entropy to define the free energy~\footnote{The use of terminology from thermodynamics is justified by Riemannian Langevin equations that naturally respect the geometry of the DLN~\cite{menon2025rle}.}
\begin{equation} \label{eq:free_energy_matrix}
    F_\beta(X) \;=\; E(X) \;-\; \frac{1}{\beta} S_N(X),
\end{equation}
and the corresponding gradient flow
\begin{equation}
    \label{eq:main-grad}
    \dot{X} = - \mathrm{grad}_{g^N} F_\beta (X), \quad X \in (\mathfrak{M}_d,g^N).
\end{equation}
Explicitly, equation~\eqref{eq:main-grad} is the matrix-valued ordinary differential equation
\begin{equation}\label{eq:dln-gf}
\dot X
\;=\;
-\,\sum_{p=1}^{N} (X X^{T})^{\frac{N-p}{N}}\; dF_\beta(X)\; (X^{T} X)^{\frac{p-1}{N}},
\end{equation}
where $dF_\beta$ denotes the differential of $F_\beta$. The analysis of this gradient flow is subtle for two reasons. First, while the vector field is continuous on $\mathbb{M}_d$ it fails to be smooth on the loci where $X$ is rank-deficient. This is why we restrict attention to $X \in \mathfrak{M}_d$. Second, while the entropy is naturally expressed in terms of singular values, the cost function for matrix sensing is not invariant under left and right rotations of $X$ and is better expressed in the standard coordinate system on $\mathbb{M}_d$, giving rise to an unwieldy system even when $d=2$.


The main new idea in this paper is to approach the gradient flow~\eqref{eq:dln-gf} using an analogy with random matrix theory (RMT). To this end, we note that the determinantal formula for $S_N$, as well as the underlying stochastic dynamics that allow us to define a thermodynamic formalism for the DLN, were based on a geometric construction of Dyson Brownian motion introduced in~\cite{huang2023motion}. Thus, the equilibria of equation~\eqref{eq:dln-gf} are analogous to the minimizers of free energy in RMT. The simplest equilibrium measures in RMT arise when we consider energies invariant under unitary transformations. Thus, the simplest setting in which we may understand the gradient flow~\eqref{eq:dln-gf} is when $E$ depends only on the singular values of $X$ in a symmetric manner. We formalize this assumption as follows:

\begin{definition}
We say that $E:\mathbb{M}_d \to \mathbb{R}$ is a {\em spectral energy\/} if it has the following form
\begin{equation}\label{eq:spectral-energy}
    E(X)=E(\sigma(X)),\qquad
    E(\sigma)=g\!\Big(\sum_{i=1}^d f(\sigma_i)\Big),
\end{equation}
where $g:\mathbb{R}\to\mathbb{R}$ is nondecreasing and $f:(0,\infty)\to\mathbb{R}$ is convex, and $\sigma=(\sigma_1,\sigma_2, \ldots, \sigma_d)$ denotes the singular values of $X$.
\end{definition}
Here and below we abuse notation somewhat, writing $E(X)$ and $E(\sigma)$ interchangeably, depending on context. No confusion should arise since we only consider spectral energies for the analysis in this paper.


Learning tasks such as matrix sensing do {\em not\/} give rise to spectral energies. However, the restriction to spectral energies provides an exactly solvable benchmark for implicit regularization in the DLN. Further, our work requires a careful analysis of the entropy formula $S_N(X)$ when the singular values are equal, providing a surprising contrast with RMT.



When $E$ is spectral, the free energy $F_\beta(\sigma)$ depends only on $\sigma$ and we may reduce the gradient flow~\eqref{eq:dln-gf} to the chamber of ordered
singular values
\begin{equation}\label{eq:sv-chamber1}
\mathcal{S}_{d}
=\{\sigma\in\mathbb{R}^d:\ \sigma_1\ge\cdots\ge\sigma_d>0\}.
\end{equation}
We denote its interior by
\begin{equation}\label{eq:sv-chamber-open}
\mathcal{S}_d^\circ
=\{\sigma\in\mathbb{R}^d:\ \sigma_1>\cdots>\sigma_d>0\}.
\end{equation}

The free energy $F_\beta$ in \eqref{eq:free_energy_matrix} for the class of spectral energies restricts to $\mathcal{S}_d$ as
\begin{equation}\label{eq:free_energy_sigma}
F_\beta(\sigma)=E(\sigma)-\frac{1}{\beta}S_N(\sigma).
\end{equation}
We equip $\mathcal{S}_d^\circ$ with the metric $g^N_\sigma$ obtained by pushing forward
$g^N$ under the singular–value map $X\mapsto \sigma(X)$ (Lemma~\ref{lem:SV-submersion}).
The resulting metric extends continuously to all of $\mathcal{S}_d$.
In Section~\ref{sec:geometry}, we show that the gradient flow on the Riemannian manifold $(\mathcal{S}^\circ_d, g^N_\sigma)$ for spectral energies is given by
\begin{equation}\label{eq:flow}
\dot\sigma_i = -\,N\,\sigma_i^{\,2-2/N}\,\partial_{\sigma_i} F_\beta(\sigma),
\qquad i=1,\dots,d,
\end{equation}
and the right-hand side extends continuously to $\mathcal{S}_d$.

We use $\mathcal{S}_d$ and its interior $\mathcal{S}_d^\circ$ interchangeably when the
distinction does not play a role. Since $\mathcal{S}_d$ is not smooth where singular values
coincide, any reference to $(\mathcal{S}_d, g^N_\sigma)$ as a Riemannian manifold is understood
to mean $(\mathcal{S}_d^\circ, g^N_\sigma)$, and smooth arguments involving the singular–value
map always take place on $\mathcal{S}_d^\circ$.

Thus, most of our analysis reduces to understanding how the gradient of the entropy affects equation~\eqref{eq:flow}. At first sight, the  entropy $S_N(X)$ is reminiscent of determinantal formulas in RMT. However, $S_N(X)$ is the {\em ratio\/} of two Vandermonde determinants and has rather different properties. In particular, it does {\em not\/} blow up when two singular values coincide. 

\begin{theorem}\label{thm:isotropic}
There exists a unique equilibrium $\sigma\in\mathcal{S}_d$ of $F_\beta$, and it has the form
\begin{equation}
\sigma_1=\cdots=\sigma_d=\sigma_\star>0,
\end{equation}
where $\sigma_\star$ is the unique solution of
\begin{equation}\label{eq:balance-equil}
g'\!\big(d\,f(\sigma_\star)\big)\,f'(\sigma_\star)
\;=\;\beta^{-1}\,\frac{d-1}{2\sigma_\star}\Big(1-\frac{1}{N}\Big).
\end{equation}
Moreover, this equilibrium is a minimizer of $F_\beta$ on $\mathcal{S}_d$.
\end{theorem}

Let us denote the equilibrium by
\begin{equation}
\label{eq:def-equilibrium}
\vec\sigma_\star=(\sigma_\star,\ldots,\sigma_\star). 
\end{equation}
The rate of relaxation to $\vec\sigma_\star$ is given by the linearization of the gradient flow~\eqref{eq:flow} at $\vec\sigma_\star$. Let $H_E=\nabla^2_\sigma E(\vec\sigma_\star)$ and $H_S=\nabla^2_\sigma S_N(\vec\sigma_\star)$ denote the Euclidean Hessians at a stationary point. Write $\theta_{\mathbf 1}(E)$ and $\theta_\perp(E)$ for the eigenvalues of $H_E$ on $\mathrm{span}\{\mathbf 1\}$ and its orthogonal complement, and define $\theta_{\mathbf 1}(S_N)$ and $\theta_\perp(S_N)$ similarly. 

\begin{theorem}\label{thm:local-rates}
The linearization of the flow \eqref{eq:flow} at $\vec\sigma_\star$ diagonalizes in the splitting 
$\mathbb{R}^d=\mathrm{span}\{\mathbf{1}\}\oplus\mathrm{span}\{\mathbf{1}\}^\perp$ with eigenvalues
\begin{align*}
\rho_{\mathbf{1}}
&=\ -\,N\,\sigma_\star^{\,2-2/N}\,\big(\theta_{\mathbf{1}}(E)-\beta^{-1}\theta_{\mathbf{1}}(S_N)\big),\\
\rho_{\perp}
&=\ -\,N\,\sigma_\star^{\,2-2/N}\,\big(\theta_{\perp}(E)-\beta^{-1}\theta_{\perp}(S_N)\big),
\qquad (\text{multiplicity }d-1).
\end{align*}
\end{theorem}

\begin{remark}[Infinite depth]
The entropy $S_N$ and the rescaled metrics $N g^N$ and $N g^N_\sigma$ have well-defined limits as $N\to\infty$.
The (renormalized) entropy is
\begin{equation} \label{eq:renormalized_entropy}
S_\infty(X)
\;=\;
\frac{1}{2}\!\sum_{1\le i<j\le d}\!
\log\!\left(
\frac{\sigma_i^2-\sigma_j^2}{\log\sigma_i^2-\log\sigma_j^2}
\right).
\end{equation}
Likewise, $N g^N$ converges to a limiting metric $g^\infty$ \cite{cohen2023deep}, and the corresponding metric $g^\infty_\sigma$ on $\mathcal{S}_d$ is well defined in the limit.
All the theorems in this paper continue to hold with these modifications in the infinite-depth regime.
\end{remark}

\subsection{Equilibria and rates on $(\mathfrak{M}_d,g^N)$}
We now describe the minimizers of the matrix gradient flow \eqref{eq:main-grad} when
$E$ is spectral. If $X(t)$ has simple singular values on an interval, then its singular
value decomposition $X(t)=U(t)\Sigma(t)V(t)^T$ varies smoothly in $t$. In these variables
the gradient flow \eqref{eq:main-grad} has an explicit form \cite[Theorem~3.2]{cohen2023deep}.
For spectral energies, the terms involving \(U\) and \(V\)
vanish identically, and hence \(\dot U=\dot V=0\). Thus \(Q:=UV^T\in O_d\) is constant, and the diagonal entries of \(\Sigma(t)\) evolve according to
\eqref{eq:flow}.

Theorem~\ref{thm:isotropic} gives a unique minimizer of $F_\beta$ on $\mathcal{S}_d$.
Since $F_\beta(X)$ depends only on the singular values of $X$, we introduce the group orbit
\begin{equation}\label{eq:orbit-def}
\mathcal O_\star := \{\sigma_\star Q : Q \in O_d\}.
\end{equation}

\begin{corollary}\label{cor:matrix-orbit}
The set of minimizers of $F_\beta$ on $\mathfrak{M}_d$ is $\mathcal O_\star$.
\end{corollary}

In particular, the limit of $X(t)$ is
\begin{equation}
X_\star=\sigma_\star\,UV^T\in\mathcal O_\star,
\end{equation}
the point of the orbit determined by the singular vectors of $X(t)$.

Since $F_\beta$ is constant on $\mathcal O_\star$, the linearization of
\eqref{eq:main-grad} at $X_\star$ vanishes on $T_{X_\star}\mathcal O_\star$.
Writing $X_\star=\sigma_\star Q$, the tangent space is
\begin{equation}\label{eq:tangent-orbit}
T_{X_\star}\mathcal O_\star=\{X_\star A:\ A^T=-A\}.
\end{equation}
The orthogonal complement of $T_{X_\star}\mathcal O_\star$ splits into the scaling
direction $\mathrm{span}\{X_\star\}$ and the subspace
\(
\{Q S : S^T = S,\ \mathrm{tr}\,S = 0\}.
\)

\begin{corollary}\label{cor:matrix-local-rates}
The linearization of the flow \eqref{eq:main-grad} at $X_\star$ diagonalizes in the splitting
\begin{equation}\label{eq:matrix-linear-splitting}
\mathbb M_d
=
T_{X_\star}\mathcal O_\star
\;\oplus\;
\mathrm{span}\{X_\star\}
\;\oplus\;
\{Q S : S^T = S,\ \mathrm{tr}\,S = 0\},
\end{equation}
with eigenvalues
\begin{align}
0 &\quad\text{on } T_{X_\star}\mathcal O_\star,\\
\rho_{\mathbf 1} &\quad\text{on } \mathrm{span}\{X_\star\},\\
\rho_\perp &\quad\text{on } \{Q S : S^T = S,\ \mathrm{tr}\,S = 0\},
\end{align}
where $\rho_{\mathbf 1}$ and $\rho_\perp$ are given in Theorem~\ref{thm:local-rates}.
\end{corollary}

\subsection{Gradient flow of the Schatten energy}
We may analyze the Schatten energies
\begin{equation}
\label{eq:schatten-energy}
  E_p(X) = \frac{1}{p}\sum_{i=1}^d \sigma_i^p, \quad 1\leq p < \infty.
\end{equation}
to provide more insight into Theorem~\ref{thm:isotropic} and Theorem~\ref{thm:local-rates}. First, in relation to Theorem~\ref{thm:isotropic} we find that $E_p$ corresponds to $g(s)=s$, $f(\sigma)=\sigma^p/p$, yielding
\begin{equation}\label{eq:eq-scale}
\sigma_\star=\left(\frac{d-1}{2\beta}\Big(1-\frac{1}{N}\Big)\right)^{\!1/p}.
\end{equation}
We may also solve for the time-dynamics explicitly. Let
\begin{equation}
\mathcal{D}
=\{\sigma\in \mathcal{S}_d:\ \sigma_1=\cdots=\sigma_d\},
\end{equation}
denote the subset of $\mathcal{S}_d$ where all singular values coincide.  
Writing $\sigma_i=s^N$ with $s>0$, the flow restricted to $\mathcal D$ becomes an ODE for a single scale variable $s$.

We write the quadrature in terms of the hypergeometric function.
Define
\begin{equation}\label{eq:T-def}
\mathcal{T}(s)
=\frac{s^{2}}{2s_\star^{\,\nu}}\;
{}_2F_1\!\left(1,\frac{2}{\nu};1+\frac{2}{\nu};\left(\frac{s}{s_\star}\right)^{\nu}\right),
\end{equation}
where ${}_2F_1$ denotes the Gauss hypergeometric function \cite{olver2010nist}.

\begin{theorem}[Exact solution on $\mathcal{D}$]\label{thm:diagonal}
Along $\mathcal{D}$ the variable $s(t)$ satisfies
\begin{equation}\label{eq:s-ode}
\dot s=-\,s^{\,\nu-1}+\frac{s_\star^{\,\nu}}{s}.
\end{equation}
Every solution of \eqref{eq:s-ode} obeys the quadrature
\begin{equation}\label{eq:s-quadrature}
t-t_0=\mathcal{T}\!\big(s(t)\big)-\mathcal{T}(s_0),
\qquad s(t_0)=s_0>0.
\end{equation}
\end{theorem}

Finally, we note that the equilibrium $\sigma_\star$ may also be understood via the following constrained (dual) entropy maximization problem. Consider
\begin{equation}\label{eq:dual-problem}
\max_X\; S_N(X)
\qquad
\text{subject to}\qquad
E_p(X)=1.
\end{equation}
At a maximizer with singular values $\vec\sigma_\star=(\sigma_\star,\dots,\sigma_\star)$, the Lagrange multiplier condition reads
\begin{equation}
\frac{d-1}{2\sigma_\star}\!\left(1-\frac{1}{N}\right)
=\lambda\,\sigma_\star^{p-1}.
\end{equation}
The constraint $E_p(X_\star) = \frac{d}{p}\sigma_\star^p=1$, fixes
\begin{equation}
\sigma_\star=\left(\frac{p}{d}\right)^{1/p},
\end{equation}
and therefore the Lagrange multiplier is
\begin{equation}\label{eq:dual-lambda}
\lambda_\star
=\frac{d-1}{2\sigma_\star^p}\!\left(1-\frac{1}{N}\right)
=\frac{1}{p}\binom{d}{2}\!\left(1-\frac{1}{N}\right).
\end{equation}
Hence the maximizers (unique up to orthogonal factors) are
\begin{equation}\label{eq:dual-solution}
X_\star=\sigma_\star\,Q,\qquad
\sigma_\star=\left(\frac{p}{d}\right)^{1/p},\quad Q\in O_d.
\end{equation}

\subsection{Concavity of the entropy}
While our approach in this paper is strongly guided by random matrix theory, Theorem~\ref{thm:isotropic} reveals subtle differences between the entropy $S_N(\sigma)$ and the analogous term in RMT. For these reasons, we record the regularity properties of $S_N(\sigma)$ separately. 

The chamber $\mathcal{S}_d$ includes points with repeated singular values (see equation~\eqref{eq:sv-chamber1}). But we still have
\begin{theorem}\label{thm:analyticity}
The entropy $S_N$ is real-analytic on $\mathcal{S}_d$. 
\end{theorem}

Let $(\mathcal{S}_d,\iota)$ denote the Riemannian manifold obtained by equipping $\mathcal{S}_d$ with the Euclidean metric on $\mathbb{R}^d$. We also note an unusual distinction between concavity of $S_N$ on $(\mathcal{S}_d,\iota)$ and the Riemannian manifold $(\mathcal{S}_d,g^N_\sigma)$.


\begin{theorem}\label{thm:concavity-gap}
The entropy $S_N$ is strictly concave on $(\mathcal{S}_d,\iota)$, except in the case $(N,d)=(2,2)$ where its Hessian has rank one. 
\end{theorem}

\begin{theorem}\label{thm:concavity-gap2}
The entropy $S_N$ is not concave on $(\mathcal{S}_d,g^N_\sigma)$: at every point with $\sigma_1=\cdots=\sigma_d$ the Hessian is indefinite.
\end{theorem}
The reader should note that the Hessian in each of these theorems is computed with respect to the metric stated in the theorem. 

\subsection{Organization of the paper}
We review the Riemannian metric $g^N$ on $\mathfrak{M}_d$, compute its restriction by Riemannian submersion $(\mathcal{S}_d,g^N_\sigma)$, and obtain the gradient flow for singular values~\eqref{eq:flow} in Section~\ref{sec:geometry}. The proofs of Theorem~\ref{thm:isotropic} and Theorem~\ref{thm:local-rates} require a careful analysis of the entropy when the singular values coincide. Thus, we study the analyticity of the entropy next in Section~\ref{sec:analyticity}. Theorem~\ref{thm:concavity-gap} is proved in Section~\ref{sec:euclid-concavity} through a pairwise block decomposition and a definiteness argument. This is followed by the proof of Theorem~\ref{thm:concavity-gap2} in Section~\ref{sec:riem-nonconcavity}. The equilibria of the free energy and the linearization of the gradient flow is established in Section~\ref{sec:equilibria-rates}. We reduce the dynamics to the scale variable $s$ and integrate the resulting equation in closed form in Section~\ref{sec:exact-flow}. We conclude with a brief discussion in Section~\ref{sec:discussion}.

\section{Riemannian Geometry of the Singular-Value Chamber}\label{sec:geometry}
\subsection{Overview}
We review the DLN metric $g^N$ and obtain the induced metric $g^N_\sigma$ on $\mathcal{S}_d$ from the singular–value map, a Riemannian submersion (Lemma \ref{lem:SV-submersion}).
We then use $g^N_\sigma$ to compute the gradient flow~\eqref{eq:flow} for spectral free energies in Lemma~\ref{lem:GF-pushforward}.


\subsection{Background}
The results in this section follow~\cite{bah2022learning,menon2025entropy}. The parameter space for the DLN is $\mathbb{M}_d^N$. Given parameters $\mathbf W=(W_N,\dots,W_1)\in\mathbb{M}_d^N$ we define the end‑to‑end matrix through the map
\begin{equation}
\phi(\mathbf W):=W_N\cdots W_1=X\in\mathbb{M}_d.    
\end{equation}
The (full-rank) balanced manifold is defined by
\begin{equation}
\mathcal{M}=\Big\{\mathbf W\in\mathbb{M}_d^N:\ \operatorname{rank}(W_p)=d\ \ \text{and}\ \ W_{p+1}^T W_{p+1}=W_p W_p^T\ \ \text{for }\ p=1,\dots,N-1\Big\}.   
\end{equation}
We use the Frobenius norm 
\[ \|\mathbf{W}\|_2^2 = \sum_{p=1}^N \mathrm{Tr}(W_p^TW_p)\]
on $\mathbb{M}_d^N$ and equip $\mathcal{M}$ with the Riemannian metric $\iota$ induced by its embedding in $(\mathbb{M}_d^N, \|\cdot\|_2^2)$.

The metric $g^N$ on $\mathfrak{M}_d$ is defined as follows. Given $X\in\mathfrak{M}_d$, define the linear operator $\mathcal{A}_{N,X}:T_X \mathfrak{M}_d^* \to T_X \mathfrak{M}_d$ by
\begin{equation}
\mathcal{A}_{N,X}(P):=\sum_{p=1}^N (XX^{T})^{\frac{N-p}{N}}\, P\, (X^{T}X)^{\frac{p-1}{N}}.    
\end{equation}
We then define 
\begin{equation}\label{eq:def-gN}
\|Z\|_{g^N}^2 \ =\ \operatorname{Tr}\!\big(Z^{T} \mathcal{A}_{N,X}^{-1} Z\big),
\qquad Z\in T_X\mathfrak{M}_d.    
\end{equation}
This metric may be described explicitly using the following
\begin{lemma}[\cite{menon2024geometry}] \label{lem:spectral}
    Let $X = U \Sigma V^T$ be the SVD of $X$. The operator $\mathcal{A}_{N,X}:T_X \mathfrak{M}_d^* \to T_X \mathfrak{M}_d$ is symmetric and positive definite with respect to the Frobenius inner-product. It has the spectral decomposition
    \begin{equation}
        \mathcal{A}_{N,X} \, u_k v_l^T = \frac{\sigma^2_k - \sigma^2_l}{\sigma^{2/N}_k - \sigma^{2/N}_l} u_k v_l^T, \quad 1 \leq k,l \leq d,
    \end{equation}
    when $k \neq l$ and 
    \begin{equation}
        \mathcal{A}_{N,X} \, u_k v_k^T = N \sigma_k^{2 - \frac{2}{N}} u_k v_k^T, \quad 1 \leq k \leq d
    \end{equation}
    where $u_k, v_l$ are the columns of $U, V$ respectively.
\end{lemma}
The explicit representation in Lemma~\ref{lem:spectral} has a simple geometric origin.
\begin{theorem}[{\cite{menon2025entropy}}]\label{thm:submersion}
The map
\begin{equation}
\phi:(\mathcal{M},\iota)\ \longrightarrow\ (\mathfrak{M}_d,g^N)    
\end{equation}
is a Riemannian submersion.
\end{theorem}

\subsection{Pushforward metric on the chamber}\label{subsec:submersion}

We work on the regular set where all singular values are simple,
\begin{equation}
\mathfrak{M}_{\mathrm{reg}}
=\{X\in\mathfrak{M}_d:\ \sigma_1(X)>\cdots>\sigma_d(X)>0\},
\end{equation}
on which the singular–value map takes values in $\mathcal{S}_d^\circ$.

With $g^N$ as in \eqref{eq:def-gN}, its pushforward to $\mathcal{S}_d^\circ$ is
\begin{equation}\label{eq:chamber-metric}
g^{N}_{\sigma}(\dot\sigma,\dot\sigma')
= \sum_{i=1}^d \frac{1}{N}\,\sigma_i^{\,2/N-2}\ \dot\sigma_i\,\dot\sigma'_i.
\end{equation}

\begin{lemma}\label{lem:SV-submersion}
The singular-value map
\begin{equation}
\sigma:(\mathfrak{M}_{\mathrm{reg}},g^N)\ \longrightarrow\ (\mathcal{S}_d^\circ,g^{N}_{\sigma}), 
\qquad X \mapsto (\sigma_1(X),\dots,\sigma_d(X)),
\end{equation}
is a Riemannian submersion.
\end{lemma}

\begin{proof}[Proof of Lemma \ref{lem:SV-submersion}]
Let $X\in\mathfrak{M}_{\mathrm{reg}}$ and write a singular value decomposition $X=U\Sigma V^T$ (so the singular values are distinct).
Set $E_{k\ell}:=u_k v_\ell^T$.
By Lemma~\ref{lem:spectral},
\begin{equation}
\mathcal{A}_{N,X}E_{k\ell}
=\begin{cases}
\displaystyle \frac{\sigma_k^2-\sigma_\ell^2}{\sigma_k^{2/N}-\sigma_\ell^{2/N}}\,E_{k\ell}, & k\neq \ell,\\[1.1ex]
\displaystyle N\,\sigma_k^{\,2-2/N}\,E_{kk}, & k=\ell,
\end{cases}    
\end{equation}
so $\mathcal{A}_{N,X}^{-1}E_{k\ell}=\mu_{k\ell}E_{k\ell}$ with $\mu_{kk}=\frac{1}{N}\sigma_k^{\,2/N-2}$ and
$\mu_{k\ell}=\frac{\sigma_k^{2/N}-\sigma_\ell^{2/N}}{\sigma_k^2-\sigma_\ell^2}$ for $k\neq\ell$.
Thus $T_X\mathfrak{M}_{\mathrm{reg}}$ decomposes as
\begin{equation}
T_X\mathfrak{M}_{\mathrm{reg}}\;=\;
\underbrace{\mathrm{span}\{E_{kk}\}_{k=1}^d}_{\mathcal{H}_X}
\ \oplus\ 
\underbrace{\mathrm{span}\{E_{k\ell}:k\neq \ell\}}_{\mathcal{V}_X},
\end{equation}
and $\mathcal{H}_X$ and $\mathcal{V}_X$ are $g^N$–orthogonal.

The first-order perturbation formula for simple singular values gives
$\mathrm{d}\sigma_k(X)[Z]=u_k^T Z v_k$ \cite[Theorem~II–5.4]{kato2013perturbation}. 
Hence $\ker\mathrm{d}\sigma(X)=\mathcal{V}_X$, and $\mathrm{d}\sigma(X)$ maps $\mathcal{H}_X$ isomorphically
onto $T_{\sigma(X)}\mathcal{S}_d^\circ\cong\mathbb{R}^d$ since $\mathrm{d}\sigma(X)[E_{kk}]=e_k$.
Therefore $\sigma:\mathfrak{M}_{\mathrm{reg}}\to\mathcal{S}_d^\circ$ is a smooth submersion.

For $\dot\sigma,\dot\sigma'\in T_{\sigma(X)}\mathcal{S}_d^\circ\cong\mathbb{R}^d$, the horizontal lifts are
$Z^{\mathrm{hor}}=\sum_i \dot\sigma_i E_{ii}=U\,\mathrm{diag}(\dot\sigma)\,V^T$ and similarly for $\dot\sigma'$.
Using $g^N(X)(E_{ii},E_{jj})=\mu_{ii}\,\delta_{ij}$, we obtain
\begin{equation}
g^N(X)\big(Z^{\mathrm{hor}},(Z')^{\mathrm{hor}}\big)
=\sum_{i=1}^d \mu_{ii}\,\dot\sigma_i\,\dot\sigma'_i
=\sum_{i=1}^d \frac{1}{N}\,\sigma_i^{\,2/N-2}\,\dot\sigma_i\,\dot\sigma'_i
= g^{N}_{\sigma(X)}(\dot\sigma,\dot\sigma').
\end{equation}
Thus $\mathrm{d}\sigma(X):(\mathcal{H}_X,g^N)\to (T_{\sigma(X)}\mathcal{S}_d^\circ,g^{N}_{\sigma(X)})$
is an isometry, which is precisely the Riemannian submersion condition.
The choice of $U,V$ does not affect $\mathcal{H}_X$ or the value of $g^{N}_{\sigma(X)}$:
when singular values are simple, the vectors $u_k,v_k$ are unique up to signs, and $\mathrm{span}\{u_k v_k^T\}$ is sign–invariant.
\end{proof}

\begin{remark}
The metric \eqref{eq:chamber-metric} extends continuously from $\mathcal{S}_d^\circ$ to all of $\mathcal{S}_d$.
At points where $\sigma_i=\sigma_j$ for some $i\neq j$, the ordered singular-value map is not smooth, so
Lemma~\ref{lem:SV-submersion} applies only on $\mathfrak{M}_{\mathrm{reg}}$.
\end{remark}

\subsection{Gradient flow on the chamber}\label{subsec:pushforward-gradient}
\begin{lemma}\label{lem:GF-pushforward}
On $(\mathcal{S}_d,g^N_{\sigma})$ the gradient flow of $F_\beta$ in \eqref{eq:free_energy_sigma} has components
\begin{equation}
\dot\sigma_i = -\,N\,\sigma_i^{\,2-2/N}\,\partial_{\sigma_i} F_\beta(\sigma),
\qquad i=1,\dots,d.
\end{equation}
\end{lemma}

Writing $\Sigma = \diag(\sigma_1,\dots,\sigma_d)$, the flow \eqref{eq:flow} can be written in matrix form as
\begin{equation}\label{eq:gradflow-matrix}
\dot\Sigma
=\ -\,N\,\Sigma^{\,2-2/N}\,\diag\!\Big(\partial_{\sigma_i} F_\beta(\sigma)\Big).
\end{equation}

\begin{proof}[Proof of Lemma \ref{lem:GF-pushforward}]
By definition of the gradient, for any $\xi\in T_\sigma\mathcal{S}_{d}\cong\mathbb{R}^d$,
\begin{equation}
g^N_{\sigma}\big(\mathrm{grad}_{g^N_{\sigma}} F_\beta,\xi\big)
= \mathrm{d}F_\beta(\sigma)[\xi]
= \sum_{i=1}^d \frac{\partial F_\beta}{\partial\sigma_i}\,\xi_i.    
\end{equation}
Since $g^N_{\sigma}$ is diagonal with $g_{ii}=\frac{1}{N}\sigma_i^{\,2/N-2}$, its inverse has
$g^{ii}=N\,\sigma_i^{\,2-2/N}$. Therefore
\begin{equation}
\mathrm{grad}_{g^N_{\sigma}} F_\beta
= \big(g^{ii}\,\partial_{\sigma_i}F_\beta\big)_{i=1}^d
= \big(N\,\sigma_i^{\,2-2/N}\,\partial_{\sigma_i}F_\beta\big)_{i=1}^d,    
\end{equation}
and the gradient flow $\dot\sigma=-\mathrm{grad}_{g^N_{\sigma}} F_\beta$ is as stated.
\end{proof}

The geometric structure of this flow is illustrated in the phase portraits of 
Figure~\ref{fig:d2d3_flows}, which visualize the trajectories 
of $\dot{\sigma}=-\grad_{g^N}F_\beta$ within $\mathcal{S}_{d}$ for $d=2$ and $d=3$.

\begin{figure}[h!]
    \centering
    \begin{subfigure}[t]{0.48\textwidth}
        \centering
        \includegraphics[width=\textwidth]{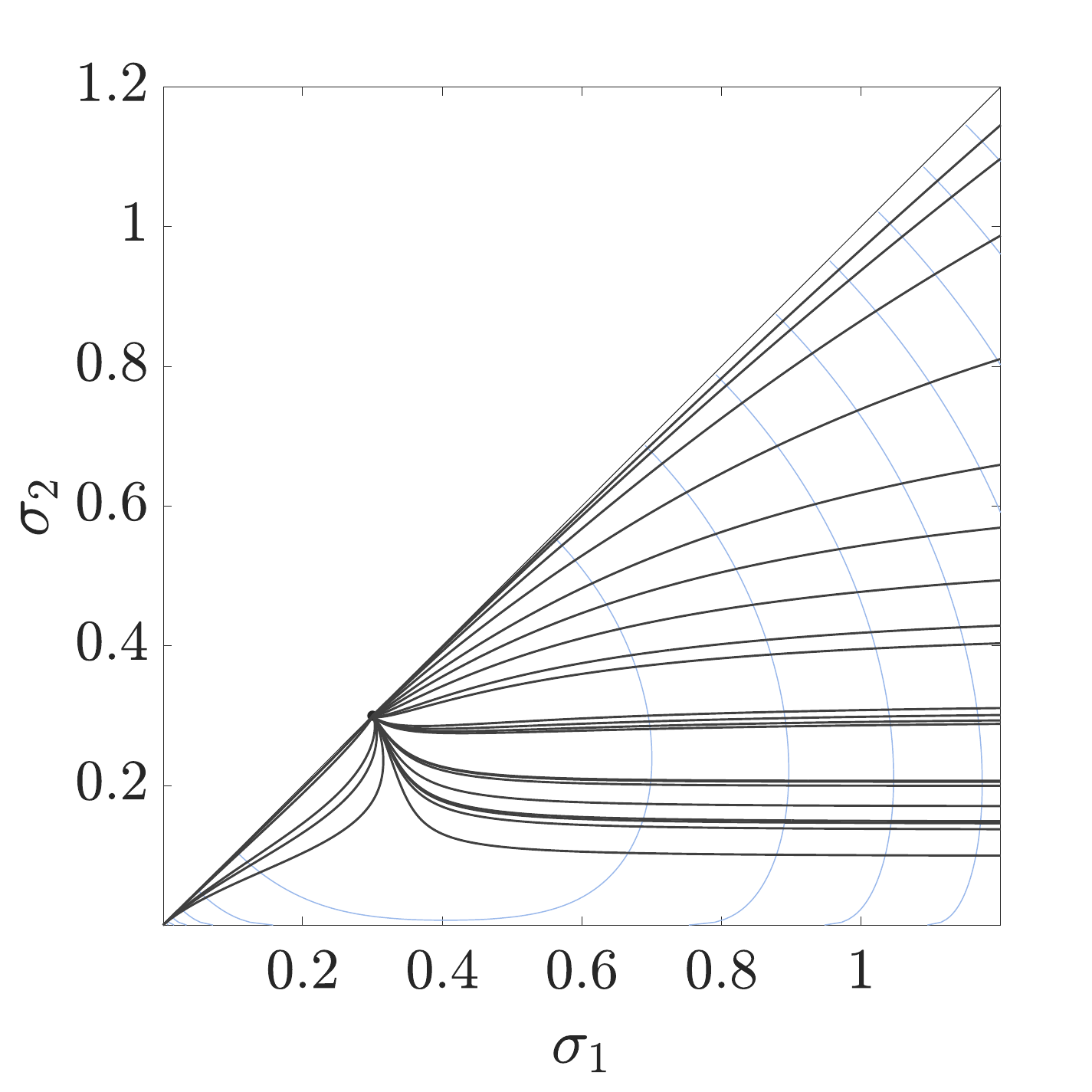}
        \caption{Gradient flow in the chamber $\sigma_1 > \sigma_2$ for $d=2$, 
        with $E(\sigma)=\tfrac{1}{p}\sum_i\sigma_i^p$. 
        Integral curves (black) are trajectories of $\dot{\sigma}=-\grad_{g^N}F_\beta$, 
        overlaid on level sets (blue) of $F_\beta(\sigma)$, converging to $\sigma_1 = \sigma_2$.}
        \label{fig:d2flow}
    \end{subfigure}
    \hfill
    \begin{subfigure}[t]{0.48\textwidth}
        \centering
        \includegraphics[width=\textwidth]{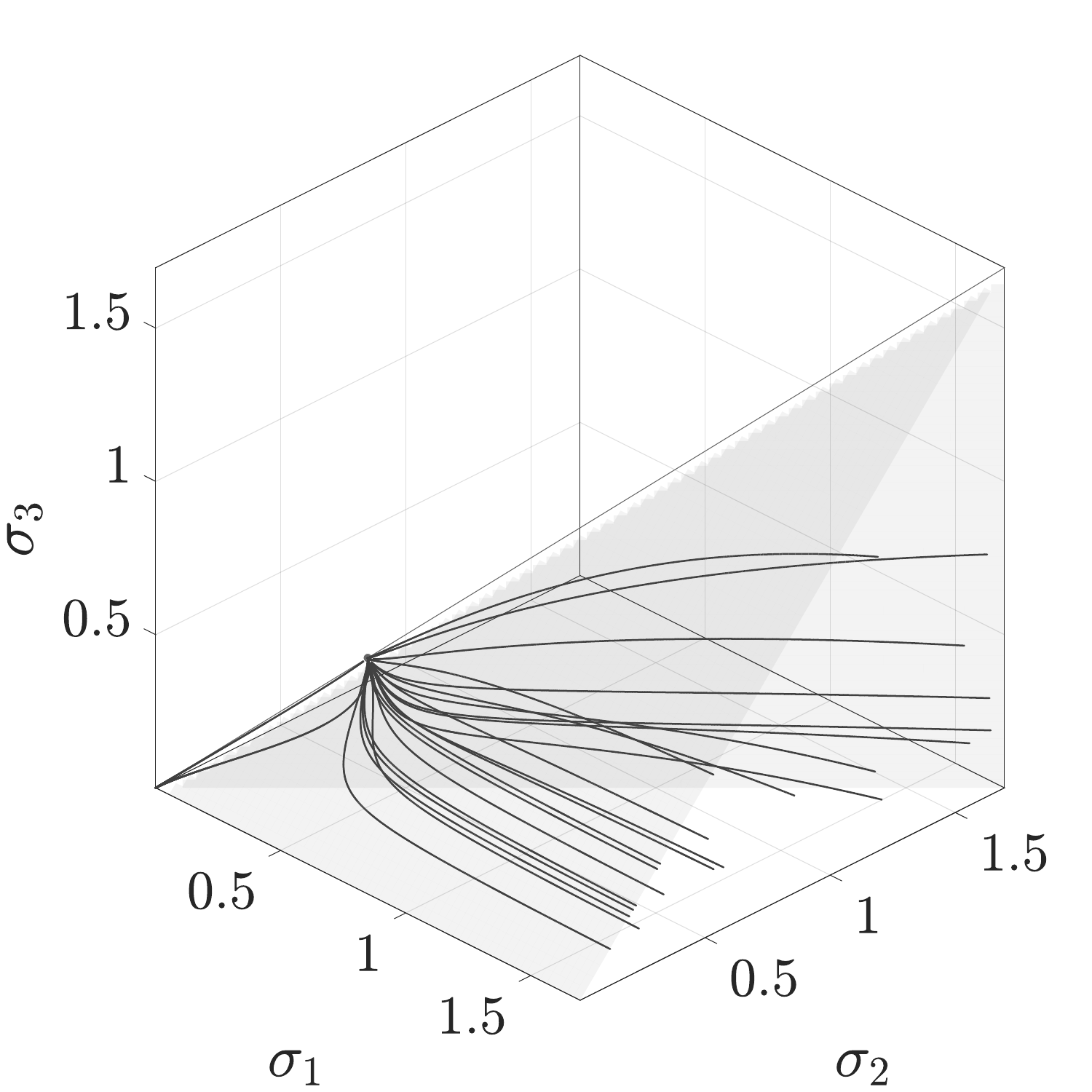}
        \caption{Gradient flow in the chamber $\sigma_1 > \sigma_2 > \sigma_3$ for $d=3$, 
        with $E(\sigma)=\tfrac{1}{p}\sum_i\sigma_i^p$. 
        Trajectories (black) evolve within the chamber bounded by 
        $\sigma_1=\sigma_2$ and $\sigma_2=\sigma_3$ (gray), 
        converging to $\sigma_1=\sigma_2=\sigma_3$.}
        \label{fig:d3flow}
    \end{subfigure}
    \caption{Phase portraits of the gradient flow 
    $\dot{\sigma}=-\grad_{g^N}F_\beta$, using the Schatten--$p$ energy $E(\sigma)=\tfrac{1}{p}\sum_i\sigma_i^p$, for $(N,p,\beta)=(10,2,5)$.}
    \label{fig:d2d3_flows}
\end{figure}

\section{Real-Analyticity of Entropy}\label{sec:analyticity}

\begin{proof}[Proof of Theorem~\ref{thm:analyticity}]
Write
\begin{equation}\label{eq:lambda-def}
\lambda_i = \sigma_i^{1/N}, \qquad i=1,\dots,d,
\end{equation}
so that $\lambda_i>0$ whenever $\sigma_i>0$.
In these variables the entropy has the representation
\begin{equation}\label{eq:SN-lambda}
S_N(\lambda)
= \widetilde{C}_N + \frac{1}{2}\sum_{1\le j<k\le d}
\log\!\left(\frac{\lambda_j^{2N}-\lambda_k^{2N}}{\lambda_j^{2}-\lambda_k^{2}}\right),
\end{equation}
where $\widetilde{C}_N$ depends only on $N$ and $d$.
Introduce
\begin{equation}\label{eq:PhiN-def}
\Phi_N(a,b) := \frac{a^{2N}-b^{2N}}{a^2-b^2}, \qquad a,b>0,
\end{equation}
so that \eqref{eq:SN-lambda} can be written as
\begin{equation}\label{eq:SN-lambda-Phi}
S_N(\lambda)
=\widetilde{C}_N+\frac{1}{2}\sum_{1\le j<k\le d}\log \Phi_N(\lambda_j,\lambda_k).
\end{equation}

The quotient in \eqref{eq:PhiN-def} satisfies the algebraic identity
\begin{equation}\label{eq:PhiN-poly}
\frac{a^{2N}-b^{2N}}{a^2-b^2}
=\sum_{m=0}^{N-1} a^{2(N-1-m)}\,b^{2m},
\end{equation}
valid for all $a,b\in\mathbb{R}$. Thus $\Phi_N$ is a polynomial in $(a,b)$ and
hence real-analytic on $\mathbb{R}^2$. In particular,
\begin{equation}\label{eq:PhiN-diag}
\Phi_N(a,a)=N\,a^{2N-2},
\end{equation}
so there is no singularity at $a=b$. 

For $a,b>0$, every term in \eqref{eq:PhiN-poly} is nonnegative and at least
one is strictly positive, so
\begin{equation}\label{eq:PhiN-positive}
\Phi_N(a,b)>0\qquad\text{for all }a,b>0.
\end{equation}
The logarithm is real-analytic on $(0,\infty)$, hence the map
\begin{equation}\label{eq:logPhi-analytic}
(a,b)\longmapsto \log\Phi_N(a,b)
\end{equation}
is real-analytic on $(0,\infty)^2$. Therefore each term
$\log\Phi_N(\lambda_j,\lambda_k)$ in \eqref{eq:SN-lambda-Phi} is
real-analytic on $(0,\infty)^d$, and finite sums preserve real-analyticity.
Thus $S_N(\lambda)$ is real-analytic for all $\lambda\in(0,\infty)^d$.

The change of variables $\sigma_i=\lambda_i^{N}$ is real-analytic on
$(0,\infty)^d$ in each coordinate. Since $\mathcal{S}_d\subset(0,\infty)^d$,
it follows that $S_N$ is real-analytic on $\mathcal{S}_d$.

Finally, the polynomial identity \eqref{eq:PhiN-poly} shows that $\Phi_N$ is
analytic at $a=b>0$, so the expression \eqref{eq:SN-lambda-Phi} extends
real-analytically to points where $\lambda_j=\lambda_k>0$. Via the change of
variables $\sigma_i=\lambda_i^N$, this gives a real-analytic extension of
$S_N$ across the sets $\sigma_i=\sigma_j>0$.
\end{proof}

\subsection{Gradient of the entropy}

\begin{lemma}\label{lem:grad-SN-sigma}
The gradient of $S_N$ has components
\begin{equation}\label{eq:gradSN-sigma}
\frac{\partial S_N}{\partial \sigma_i}
=\sum_{k\ne i}\left(
\frac{\sigma_i}{\sigma_i^2-\sigma_k^2}
-\frac{\sigma_i^{2/N-1}}{N\big(\sigma_i^{2/N}-\sigma_k^{2/N}\big)}\right).
\end{equation}
For each fixed $i$ and $k\ne i$ the summand has a finite limit as
$\sigma_k\to\sigma_i=\sigma_\star>0$, namely
\begin{equation}\label{eq:gradSN-sigma-limit}
\frac{\sigma_i}{\sigma_i^2-\sigma_k^2}
-\frac{\sigma_i^{2/N-1}}{N\big(\sigma_i^{2/N}-\sigma_k^{2/N}\big)}
\longrightarrow \frac{1}{2\sigma_\star}\Big(1-\frac{1}{N}\Big).
\end{equation}
\end{lemma}

\begin{proof}
We start from the representation \eqref{eq:SN-lambda} in the variables
$\lambda_i=\sigma_i^{1/N}$,
\begin{equation}\label{eq:SN-lambda-again}
S_N(\lambda)
= \widetilde{C}_N + \frac{1}{2}\sum_{1\le j<k\le d}
\log\!\left(\frac{\lambda_j^{2N}-\lambda_k^{2N}}{\lambda_j^{2}-\lambda_k^{2}}\right),
\end{equation}
valid for $\lambda_i>0$. Differentiating \eqref{eq:SN-lambda-again} with respect
to $\lambda_i$ and noting that only pairs containing $i$ contribute gives
\begin{equation}\label{eq:gradSN-lambda}
\frac{\partial S_N}{\partial \lambda_i}
=\sum_{k\ne i}\left(
N\,\frac{\lambda_i^{2N-1}}{\lambda_i^{2N}-\lambda_k^{2N}}
-\frac{\lambda_i}{\lambda_i^2-\lambda_k^2}\right).
\end{equation}

The change of variables $\sigma_i=\lambda_i^N$ implies
\begin{equation}\label{eq:chain-rule}
\frac{\partial S_N}{\partial \sigma_i}
=\frac{1}{N\lambda_i^{N-1}}\frac{\partial S_N}{\partial \lambda_i}
=\lambda_i^{1-N}\frac{\partial S_N}{\partial \lambda_i}.
\end{equation}
Substituting \eqref{eq:gradSN-lambda} into \eqref{eq:chain-rule} and using
$\sigma_i=\lambda_i^N$ yields
\begin{align}\label{eq:gradSN-sigma-derivation}
\frac{\partial S_N}{\partial \sigma_i}
&=\sum_{k\ne i}\left(
\lambda_i^{1-N}\,N\,\frac{\lambda_i^{2N-1}}{\lambda_i^{2N}-\lambda_k^{2N}}
-\lambda_i^{1-N}\,\frac{\lambda_i}{\lambda_i^2-\lambda_k^2}\right)\\
&=\sum_{k\ne i}\left(
N\,\frac{\lambda_i^{N}}{\lambda_i^{2N}-\lambda_k^{2N}}
-\frac{\lambda_i^{2-N}}{\lambda_i^2-\lambda_k^2}\right).\nonumber
\end{align}
Replacing $\lambda_i^N$ and $\lambda_k^N$ by $\sigma_i$ and $\sigma_k$ in
\eqref{eq:gradSN-sigma-derivation} gives exactly \eqref{eq:gradSN-sigma}.

For the limit \eqref{eq:gradSN-sigma-limit}, set $\lambda_i=\lambda_\star$ and
$\lambda_k=\lambda_\star(1-\varepsilon)$ with $\varepsilon\downarrow 0$. Then
\begin{equation}\label{eq:binom-expansion}
(1-\varepsilon)^{2N}
=1-2N\varepsilon+N(2N-1)\varepsilon^2+O(\varepsilon^3).
\end{equation}
Using \eqref{eq:gradSN-lambda} and \eqref{eq:chain-rule} at
$\lambda_i=\lambda_\star$ and $\lambda_k=\lambda_\star(1-\varepsilon)$,
a direct expansion of each term gives
\begin{equation}\label{eq:pair-expansion}
N\,\frac{\lambda_i^{2N-1}}{\lambda_i^{2N}-\lambda_k^{2N}}
=\frac{1}{2\lambda_\star\varepsilon}+\frac{2N-1}{4\lambda_\star}
+O(\varepsilon),\qquad
\frac{\lambda_i}{\lambda_i^2-\lambda_k^2}
=\frac{1}{2\lambda_\star\varepsilon}+\frac{1}{4\lambda_\star}
+O(\varepsilon).
\end{equation}
Subtracting these expressions cancels the $1/\varepsilon$ term and yields
\begin{equation}\label{eq:lambda-pair-limit}
N\,\frac{\lambda_i^{2N-1}}{\lambda_i^{2N}-\lambda_k^{2N}}
-\frac{\lambda_i}{\lambda_i^2-\lambda_k^2}
\longrightarrow \frac{N-1}{2\lambda_\star}
\qquad (\varepsilon\downarrow 0).
\end{equation}
Since $\lambda_\star=\sigma_\star^{1/N}$, this is exactly
\eqref{eq:gradSN-sigma-limit} after rewriting in $\sigma_\star$.

Thus each summand in \eqref{eq:gradSN-sigma} has a finite limit as
$\sigma_k\to\sigma_i>0$, and the sum over $k\ne i$ extends continuously to
points with $\sigma_k=\sigma_i$. 
\end{proof}

For the renormalized entropy $S_\infty$ \eqref{eq:renormalized_entropy}, a similar argument gives
\begin{equation}\label{eq:gradSinf-sigma}
\frac{\partial S_\infty}{\partial \sigma_i}
=\sum_{k\ne i}\left(
\frac{\sigma_i}{\sigma_i^2-\sigma_k^2}
-\frac{\sigma_i^{-1}}{\log\sigma_i^2-\log\sigma_k^2}\right),
\end{equation}
and each summand has limit $\frac{1}{2\sigma_i}$ as $\sigma_k\to\sigma_i$.

\section{Proof of Theorem \ref{thm:concavity-gap}}\label{sec:euclid-concavity}

\subsection{Overview}
We work on the Riemannian manifold $(\mathcal{S}_d,\iota)$, where $\iota$ is the standard inner product on $\mathbb{R}^d$.
The Hessian of $S_N$ is written as a sum of $2\times2$ blocks, each depending only on a pair of singular values.
These blocks can be analyzed explicitly: they are negative definite for $N>2$ and rank‑one negative semidefinite for $N=2$.
Summing over all pairs yields Theorem~\ref{thm:concavity-gap}.

\subsection{Notation}
For a smooth $f:{\mathcal S}_{d}\to\mathbb{R}$ we write
\begin{equation}
(\nabla_\sigma f(\sigma))_i \;=\; \frac{\partial f}{\partial \sigma_i}(\sigma),\qquad
(\nabla^2_\sigma f(\sigma))_{ij} \;=\; \frac{\partial^2 f}{\partial \sigma_i\,\partial \sigma_j}(\sigma),
\end{equation}
so $\nabla_\sigma f$ and $\nabla^2_\sigma f$ are the gradient and Hessian in the coordinates $\sigma=(\sigma_1,\dots,\sigma_d)\in{\mathcal S}_{d}\subset\mathbb{R}^d$.

Definiteness is understood with respect to the standard inner product on $\mathbb{R}^d$. In particular,
\begin{equation}
v^T \nabla^2_\sigma f(\sigma)v \le 0\quad\text{for all }v\in\mathbb{R}^d
\end{equation}
means that $f$ is concave at $\sigma$ with respect to the Euclidean metric.

For a symmetric matrix $A$, we write
\begin{equation}
A\preceq 0 \quad\text{if $A$ is negative semidefinite}, 
\qquad A\prec 0 \quad\text{if $A$ is negative definite}.    
\end{equation}

\subsection{Hessian of the entropy}
We first record the Hessian in the $\sigma$–coordinates. In the next subsection it is expressed as a sum of $2\times2$ blocks.

\begin{lemma}\label{lem:hess-sigma}
For $S_N$ the second derivatives in the coordinates $\sigma$ are
\begin{equation}\label{eq:HessSN-sigma}
\frac{\partial^2 S_N}{\partial \sigma_i \partial \sigma_j } =
\begin{cases}
\displaystyle
\sum_{k\neq i}\left(
\frac{-\sigma_i^2 - \sigma_k^2}{(\sigma_i^2 - \sigma_k^2)^2}
+\frac{\sigma_i^{2/N - 2}\Big(N(\sigma_i^{2/N} - \sigma_k^{2/N}) + 2 \sigma_k^{2/N}\Big)}
{\Big(N(\sigma_i^{2/N} - \sigma_k^{2/N})\Big)^2}
\right), & i=j,\\[2ex]
\displaystyle
\frac{2\sigma_i\sigma_j}{(\sigma_i^2 - \sigma_j^2)^2}
-\frac{2 \sigma_i^{2/N - 1}\sigma_j^{2/N-1}}{\Big(N(\sigma_i^{2/N} - \sigma_j^{2/N})\Big)^2}, & i\neq j.
\end{cases}
\end{equation}
and $\nabla^2_\sigma S_N$ extends continuously to all of ${\mathcal{S}}_d$.
\end{lemma}

\begin{proof}
We start from the expression for the gradient in $\sigma$–coordinates (Lemma~\ref{lem:grad-SN-sigma}),
\begin{equation}\label{eq:gradSN-sigma-recall}
\frac{\partial S_N}{\partial \sigma_i}
=\sum_{k\ne i}\left(
\frac{\sigma_i}{\sigma_i^2-\sigma_k^2}
-\frac{\sigma_i^{2/N-1}}{N(\sigma_i^{2/N}-\sigma_k^{2/N})}\right).
\end{equation}
Differentiating the $k$–th summand in $\sigma_j$ for $j\neq i$ gives the off–diagonal entries,
\begin{equation}
\frac{\partial^2 S_N}{\partial \sigma_i\,\partial \sigma_j}
=\frac{2\sigma_i\sigma_j}{(\sigma_i^2-\sigma_j^2)^2}
-\frac{2\,\sigma_i^{2/N-1}\sigma_j^{2/N-1}}
{\big(N(\sigma_i^{2/N}-\sigma_j^{2/N})\big)^2},
\qquad i\neq j,
\end{equation}
and differentiating in $\sigma_i$ and summing over $k\neq i$ gives the diagonal entries,
\begin{equation}
\frac{\partial^2 S_N}{\partial \sigma_i^2}
=\sum_{k\ne i}\left(
\frac{-\sigma_i^2-\sigma_k^2}{(\sigma_i^2-\sigma_k^2)^2}
+\frac{\sigma_i^{2/N-2}\Big(N(\sigma_i^{2/N}-\sigma_k^{2/N})+2\sigma_k^{2/N}\Big)}
{\big(N(\sigma_i^{2/N}-\sigma_k^{2/N})\big)^2}\right),
\end{equation}
which is exactly \eqref{eq:HessSN-sigma}.

Each off–diagonal entry is $p_N(\sigma_i,\sigma_j)$ and each diagonal summand is $q_N(\sigma_i,\sigma_k)$ in the notation of \eqref{eq:pN-def}–\eqref{eq:qN-def} below. Lemma~\ref{lem:signs} shows that $p_N$ and $q_N$ have finite limits as $\sigma_k\to\sigma_i>0$, so all entries extend continuously to ${\mathcal{S}}_d$.
\end{proof}

For the renormalized entropy $S_\infty$ \eqref{eq:renormalized_entropy}, differentiating the gradient in $\sigma$–coordinates gives
\begin{equation}\label{eq:HessSInf-sigma}
\frac{\partial^2 S_\infty}{\partial \sigma_i \partial\sigma_j} =
\begin{cases}
\displaystyle
\sum_{k\neq i}\left(
\frac{-\sigma_i^2 - \sigma_k^2}{(\sigma_i^2 - \sigma_k^2)^2}
+\frac{\sigma_i^{-2}\big(\log (\sigma_i / \sigma_k) +1\big)}
{2\big(\log (\sigma_i / \sigma_k)\big)^2}
\right), & i=j,\\[2ex]
\displaystyle
\frac{2\sigma_i\sigma_j}{(\sigma_i^2-\sigma_j^2)^2}
-\frac{\sigma_i^{-1}\sigma_j^{-1}}{2 \big(\log (\sigma_i / \sigma_j)\big)^2}, & i\neq j,
\end{cases}
\end{equation}
and each summand again has a finite limit as $\sigma_j\to\sigma_i>0$, so $\nabla^2_\sigma S_\infty$ also extends continuously to ${\mathcal{S}}_d$.

\subsection{Block decomposition}

We now express the Hessian of $S_N$ as a sum of embedded $2\times2$ blocks, each depending only on a pair of singular values.  
Define the kernels
\begin{equation}\label{eq:pN-def}
p_N(a,b) := \frac{2ab}{(a^2-b^2)^2}
            - \frac{2\,a^{\frac{2}{N}-1} b^{\frac{2}{N}-1}}
                   { \big( N(a^{\frac{2}{N}}-b^{\frac{2}{N}} ) \big)^2},
\end{equation}
\begin{equation}\label{eq:qN-def}
q_N(a,b) := -\frac{a^2+b^2}{(a^2-b^2)^2}
            + \frac{a^{\frac{2}{N}-2} \big( N(a^{\frac{2}{N}}-b^{\frac{2}{N}}) 
            + 2 b^{\frac{2}{N}} \big)}
                   { \big( N(a^{\frac{2}{N}}-b^{\frac{2}{N}} ) \big)^2},
\end{equation}
and for $1\le i<j\le d$ let
\begin{equation}\label{eq:iota-ij}
\iota_{ij}:\mathbb{R}^2\hookrightarrow\mathbb{R}^d,\qquad
\iota_{ij}(u,v)=u\,e_i+v\,e_j,
\end{equation}
with
\begin{equation}\label{eq:BN-block}
B_N^{(ij)}(\sigma)=
\begin{pmatrix}
q_N(\sigma_i,\sigma_j) & p_N(\sigma_i,\sigma_j) \\
p_N(\sigma_i,\sigma_j) & q_N(\sigma_j,\sigma_i)
\end{pmatrix}.
\end{equation}

\begin{lemma}\label{lem:block-sigma}
For every $\sigma\in\mathcal S_d$,
\begin{equation}\label{eq:block-sum-sigma}
\nabla^2_\sigma S_N(\sigma)
=\sum_{1\le i<j\le d}\iota_{ij}\,B_N^{(ij)}(\sigma)\,\iota_{ij}^T.
\end{equation}
\end{lemma}

Equivalently,
\begin{equation}
(\nabla^2_\sigma S_N)_{ij}=p_N(\sigma_i,\sigma_j)\ (i\neq j),
\qquad
(\nabla^2_\sigma S_N)_{ii}=\sum_{k\neq i}q_N(\sigma_i,\sigma_k).
\end{equation}

To study each block $B_N^{(ij)}$, we rewrite $p_N$ and $q_N$ in terms of the single ratio
$r=\lambda_i/\lambda_j>1$, where $\lambda_\ell=\sigma_\ell^{1/N}$.

\begin{lemma}\label{lem:signs}
For $i<j$ and $r=\lambda_i/\lambda_j>1$,
\begin{equation}\label{eq:pN-r}
p_N(\sigma_i,\sigma_j)
=\frac{1}{\sigma_j^2}
  \left(\frac{2 r^N}{(r^{2N}-1)^2}
        -\frac{2}{N^2}\frac{r^{2-N}}{(r^2-1)^2}\right),
\end{equation}
\begin{equation}\label{eq:qN-r}
q_N(\sigma_i,\sigma_j)
=\frac{1}{\sigma_j^2}
  \left(-\frac{r^{2N}+1}{(r^{2N}-1)^2}
        +\frac{r^{2-2N}}{N^2}\frac{N(r^2-1)+2}{(r^2-1)^2}\right),
\end{equation}
and in particular $p_N(\sigma_i,\sigma_j)<0$ and $q_N(\sigma_i,\sigma_j)<0$.
As $r\downarrow 1$ (equivalently $\sigma_i\to\sigma_j=\sigma$),
\begin{equation}\label{eq:limits-pair}
p_N(\sigma_i,\sigma_j)\to -\frac{1}{6\sigma^2}\Big(1-\frac{1}{N^2}\Big),\qquad
q_N(\sigma_i,\sigma_j)\to -\frac{1}{3\sigma^2}\Big(1-\frac{3}{2N}+\frac{1}{2N^2}\Big).
\end{equation}
\end{lemma}

Since the entries of $B_N^{(ij)}$ are negative, we next determine when each block is negative definite.

\begin{lemma}\label{lem:block-def}
For $i<j$:
\begin{enumerate}
\item If $N=2$ and $\sigma_i\neq\sigma_j$, then
\begin{equation}\label{eq:N2-rank1-correct}
B_2^{(ij)}(\sigma)
=-\frac{1}{2(\sigma_i+\sigma_j)^2}
 \begin{pmatrix}1&1\\[0.3ex]1&1\end{pmatrix},
\end{equation}
and $B_2^{(ij)}$ is rank-one negative semidefinite.
\item If $N>2$ and $\sigma_i\neq\sigma_j$, then $B_N^{(ij)}(\sigma)\prec 0$.
\end{enumerate}
\end{lemma}

We now deduce the definiteness of the full Hessian from the blocks.

\begin{lemma}\label{lem:criterion}
Let
\begin{equation}
A=\sum_{1\le i<j\le d}\iota_{ij}B^{(ij)}\iota_{ij}^T
\end{equation}
with each $B^{(ij)}$ symmetric. Then:
\begin{enumerate}
\item If $B^{(ij)}\prec 0$ for all $i<j$, then $A\prec 0$.
\item If each $B^{(ij)}=-\gamma_{ij}vv^T$ with $\gamma_{ij}>0$ and $v=(1,1)^T$, then $A\preceq 0$, with strict negativity when $d\ge3$ and rank one when $d=2$.
\end{enumerate}
\end{lemma}

\begin{remark}
The decomposition 
\(
A=\sum_{i<j} \iota_{ij} B^{(ij)} \iota_{ij}^{T}
\)
reduces negativity of \(A\) to negativity of its \(2\times 2\) blocks.
Since the cone \(\{ M : M \prec 0 \}\) is convex and closed under addition,
\(B^{(ij)} \prec 0\) for all pairs implies \(A \prec 0\).

In the rank--one case \(B^{(ij)} = -\gamma_{ij} vv^{T}\) with \(v=(1,1)^{T}\),
each block lies on the boundary of the negative cone, so \(A \preceq 0\).
For \(d \ge 3\) the embedded directions \(\iota_{ij} v = e_i + e_j\) span all
of \(\mathbb{R}^{d}\). Hence their sum leaves no nontrivial kernel and the
full matrix is strictly negative.  
For \(d=2\) these directions span only a line, so the sum is rank--one negative semidefinite.
\end{remark}

\subsection{Proof of Theorem \ref{thm:concavity-gap}}

\begin{proof}
Fix $\sigma\in\mathcal S_d$.
By Lemma~\ref{lem:block-sigma}, the Hessian admits the block decomposition
\begin{equation}
\nabla^2_{\sigma} S_N(\sigma)\ =\ \sum_{1\le i<j\le d}\ \iota_{ij}\,B_N^{(ij)}(\sigma)\,\iota_{ij}^{T},
\qquad
B_N^{(ij)}(\sigma)=
\begin{pmatrix}
q_N(\sigma_i,\sigma_j) & p_N(\sigma_i,\sigma_j)\\[0.3ex]
p_N(\sigma_i,\sigma_j) & q_N(\sigma_j,\sigma_i)
\end{pmatrix},
\end{equation}
with $p_N,q_N$ as in \eqref{eq:pN-def}–\eqref{eq:qN-def}.

\emph{Case $N=2$.}
Lemma~\ref{lem:block-def} gives, for every unordered pair $\{i,j\}$ (including $\sigma_i=\sigma_j$ via the limits in Lemma~\ref{lem:signs}),
\begin{equation}
B_{2}^{(ij)}(\sigma)\ =\ -\,\frac{1}{2(\sigma_i+\sigma_j)^2}
\begin{pmatrix}1&1\\[0.2ex]1&1\end{pmatrix}
\ =:\ -\,\gamma_{ij}\,vv^T,\qquad \gamma_{ij}>0,\ v=(1,1)^T.
\end{equation}
Thus $\nabla^2_{\sigma} S_2(\sigma)$ is a sum of embedded rank–one negative semidefinite blocks of the form $-\gamma_{ij}vv^T$.  
By Lemma~\ref{lem:criterion}, the sum is negative semidefinite for all $d$.  
When $d\ge 3$, the embedded directions $\iota_{ij}v=e_i+e_j$ span $\mathbb{R}^d$, so the Hessian is negative definite.  
When $d=2$, there is a one–dimensional kernel $\mathrm{span}\{(1,-1)^T\}$ and $\nabla^2_{\sigma} S_2(\sigma)$ has rank one.  

\emph{Case $N>2$.}
First suppose $\sigma_i\neq\sigma_j$ for all $i\neq j$.  
Lemma~\ref{lem:block-def} shows that each block $B_N^{(ij)}(\sigma)$ is negative definite.  
Applying Lemma~\ref{lem:criterion} to the block sum yields
\begin{equation}
\nabla^2_{\sigma} S_N(\sigma)\ \prec\ 0
\end{equation}
at every point with distinct singular values.

It remains to treat points with $\sigma_i=\sigma_j=\sigma$ for some $i\neq j$.  
By Lemma~\ref{lem:signs}, as $\sigma_i\to\sigma_j=\sigma$ one has
\begin{equation}
p_N(\sigma_i,\sigma_j)\ \longrightarrow\ -\,\frac{1}{6\,\sigma^2}\Big(1-\frac{1}{N^2}\Big),
\qquad
q_N(\sigma_i,\sigma_j)\ \longrightarrow\ -\,\frac{1}{3\,\sigma^2}\Big(1-\frac{3}{2N}+\frac{1}{2N^2}\Big),
\end{equation}
so the limiting $2\times 2$ block is
\begin{equation}
\widehat B_N\ :=\
\begin{pmatrix}
q & p\\[0.3ex] p & q
\end{pmatrix},
\qquad
p=-\,\frac{1}{6\,\sigma^2}\Big(1-\frac{1}{N^2}\Big),\quad
q=-\,\frac{1}{3\,\sigma^2}\Big(1-\frac{3}{2N}+\frac{1}{2N^2}\Big).
\end{equation}
The eigenvalues of $\widehat B_N$ are $q\pm p$, and a direct calculation gives
\begin{equation}
q+p\ =\ -\,\frac{1}{2\sigma^2}\Big(1-\frac{1}{N}\Big)\ <\ 0,
\qquad
q-p\ =\ -\,\frac{1}{6\sigma^2}\Big(1-\frac{3}{N}+\frac{3}{N^2}\Big)\ <\ 0
\quad\text{for }N>2.
\end{equation}
Thus $\widehat B_N\prec 0$, and by continuity this is the value of $B_N^{(ij)}(\sigma)$ on $\{\sigma_i=\sigma_j\}$.  
Hence each block $B_N^{(ij)}(\sigma)$ is negative definite for all $\sigma\in\mathcal S_d$ when $N>2$.  
Lemma~\ref{lem:criterion} then implies that $\nabla^2_{\sigma} S_N(\sigma)\prec 0$ on $(\mathcal S_d,\iota)$.

Combining the two cases, we obtain that $S_N$ has negative definite Hessian on $(\mathcal S_d,\iota)$ for all $(N,d)\neq(2,2)$, and in the exceptional case $(N,d)=(2,2)$ the Hessian has rank one.  
\end{proof}

\subsection{Proofs of Lemmas}

\begin{proof}[Proof of Lemma~\ref{lem:block-sigma}]
By Lemma~\ref{lem:hess-sigma}, for $i\neq j$ one has
\begin{equation}
\frac{\partial^2 S_N}{\partial \sigma_i\,\partial \sigma_j}
=\frac{2\sigma_i\sigma_j}{(\sigma_i^2-\sigma_j^2)^2}
-\frac{2\,\sigma_i^{2/N-1}\sigma_j^{2/N-1}}
{\big(N(\sigma_i^{2/N}-\sigma_j^{2/N})\big)^2}
= p_N(\sigma_i,\sigma_j),
\end{equation}
and for $i=j$,
\begin{equation}
\frac{\partial^2 S_N}{\partial \sigma_i^2}
=\sum_{k\ne i}\left(
\frac{-\sigma_i^2-\sigma_k^2}{(\sigma_i^2-\sigma_k^2)^2}
+\frac{\sigma_i^{2/N-2}\Big(N(\sigma_i^{2/N}-\sigma_k^{2/N})+2\sigma_k^{2/N}\Big)}
{\big(N(\sigma_i^{2/N}-\sigma_k^{2/N})\big)^2}\right)
=\sum_{k\ne i}q_N(\sigma_i,\sigma_k).
\end{equation}
On the indices $\{i,j\}$ the principal $2\times 2$ block of $\nabla^2_\sigma S_N$ is therefore $B_N^{(ij)}(\sigma)$, and composing with the injections $\iota_{ij}$ gives \eqref{eq:block-sum-sigma}.
\end{proof}

\begin{proof}[Proof of Lemma~\ref{lem:signs}]
Substitute $\sigma_\ell=\lambda_\ell^N$ into Lemma~\ref{lem:block-sigma}.
For the off–diagonal entry,
\begin{equation}
p_N(\sigma_i,\sigma_j)
=\frac{2\,\lambda_i^N\lambda_j^N}{(\lambda_i^{2N}-\lambda_j^{2N})^2}
-\frac{2\,\lambda_i^{2-N}\lambda_j^{2-N}}{N^2(\lambda_i^2-\lambda_j^2)^2}.
\end{equation}
Factoring $\lambda_j$ and setting $r=\lambda_i/\lambda_j$ gives
\begin{equation}
p_N(\sigma_i,\sigma_j)
=\frac{1}{\lambda_j^{2N}}\left(\,
\frac{2\,r^{N}}{(r^{2N}-1)^2}\;-\;\frac{2}{N^2}\,\frac{r^{\,2-N}}{(r^2-1)^2}\right)
\;=\;\frac{1}{\sigma_j^2}\left(\,
\frac{2\,r^{N}}{(r^{2N}-1)^2}\;-\;\frac{2}{N^2}\,\frac{r^{\,2-N}}{(r^2-1)^2}\right),
\end{equation}
which is \eqref{eq:pN-r}.  

Writing $r=e^t$ ($t>0$) and using
\begin{equation}
r^{2m}-1=2\,e^{mt}\sinh(mt),\qquad (r^2-1)=2\,e^{t}\sinh(t),
\end{equation}
we obtain
\begin{equation}
\frac{2\,r^{N}}{(r^{2N}-1)^2}
=\frac{1}{2}\,\frac{e^{-Nt}}{\sinh^2(Nt)},
\qquad
\frac{2}{N^2}\,\frac{r^{\,2-N}}{(r^2-1)^2}
=\frac{1}{2}\,\frac{e^{-Nt}}{N^2\sinh^2 t}.
\end{equation}
Hence
\begin{equation}
p_N(\sigma_i,\sigma_j)
=\frac{e^{-Nt}}{2\,\sigma_j^2}\left(\frac{1}{\sinh^2(Nt)}-\frac{1}{N^2\sinh^2 t}\right),
\end{equation}
and $\sinh(Nt)>N\sinh t$ for $t>0$ (for instance, $\sinh x/x$ is increasing on $(0,\infty)$), so $p_N(\sigma_i,\sigma_j)<0$.

For the diagonal summand,
\begin{equation}
q_N(\sigma_i,\sigma_j)
=-\frac{\lambda_i^{2N}+\lambda_j^{2N}}{(\lambda_i^{2N}-\lambda_j^{2N})^2}
+\frac{\lambda_i^{2-2N}\big(N(\lambda_i^2-\lambda_j^2)+2\lambda_j^2\big)}
{N^2(\lambda_i^2-\lambda_j^2)^2}.
\end{equation}
The same substitution yields \eqref{eq:qN-r} after factoring $\lambda_j^{2N}=\sigma_j^2$. 
Expressing again in $t=\log r>0$ shows
\begin{equation}
q_N(\sigma_i,\sigma_j)
=\frac{e^{-Nt}}{2\,\sigma_j^2}\left(
-\frac{\cosh(Nt)}{\sinh^2(Nt)}
+\frac{e^{-Nt}}{N^2}\cdot\frac{N e^{t}\sinh t+1}{\sinh^2 t}\right),
\end{equation}
and a direct comparison using $\sinh(Nt)>N\sinh t$ and $\cosh(Nt)\ge 1$ yields strict negativity for all $t>0$ and $N\ge 2$.

The limits in \eqref{eq:limits-pair} as $r\downarrow 1$ follow by Taylor expansion.
Writing $r=e^t$ with $t\downarrow 0$ and using
\begin{align}
\begin{split}
\sinh t&=t+\tfrac{1}{6}t^3+O(t^5),\\
\sinh(Nt)&=Nt+\tfrac{N^3}{6}t^3+O(t^5),\\ 
\cosh(Nt)&=1+\tfrac{N^2}{2}t^2+O(t^4),
\end{split}
\end{align}
one obtains the stated limits after a straightforward calculation.
\end{proof}

\begin{proof}[Proof of Lemma~\ref{lem:block-def}]
(1) For $N=2$, insert $N=2$ into \eqref{eq:pN-def}–\eqref{eq:qN-def}. Using $(\sigma_i^2-\sigma_j^2)^2=(\sigma_i-\sigma_j)^2(\sigma_i+\sigma_j)^2$,
\begin{align}
\begin{split}
\frac{2\sigma_i\sigma_j}{(\sigma_i^2-\sigma_j^2)^2}-\frac{1}{2(\sigma_i-\sigma_j)^2}
\ =\ -\,\frac{1}{2(\sigma_i+\sigma_j)^2},\\
-\frac{\sigma_i^2+\sigma_j^2}{(\sigma_i^2-\sigma_j^2)^2}+\frac{1}{2(\sigma_i-\sigma_j)^2}
\ =\ -\,\frac{1}{2(\sigma_i+\sigma_j)^2},
\end{split}
\end{align}
so $p_2(\sigma_i,\sigma_j)=q_2(\sigma_i,\sigma_j)=-1/(2(\sigma_i+\sigma_j)^2)$, which gives \eqref{eq:N2-rank1-correct}.

(2) For $N>2$, Lemma~\ref{lem:signs} gives
\[
p_N(\sigma_i,\sigma_j)<0,\qquad q_N(\sigma_i,\sigma_j)<0,\qquad q_N(\sigma_j,\sigma_i)<0,
\]
so $\mathrm{tr}\,B_N^{(ij)}<0$. It remains to show $\det B_N^{(ij)}>0$.

Set $r:=\lambda_i/\lambda_j>1$ and factor the common positive scale $\sigma_j^{-4}$ to write
\begin{equation}
\det B_N^{(ij)}\ =\ \frac{1}{\sigma_j^4}\,\Delta_N(r),\qquad
\Delta_N(r):=q_N(r,1)\,q_N(1,r)-\big(p_N(r,1)\big)^2.
\end{equation}
From the limits in Lemma~\ref{lem:signs} (letting $r\downarrow 1$) we obtain
\begin{equation}
\Delta_N(1)\ =\ \bigg(-\frac{1}{3}\Big(1-\frac{3}{2N}+\frac{1}{2N^2}\Big)\bigg)^2
-\bigg(-\frac{1}{6}\Big(1-\frac{1}{N^2}\Big)\bigg)^2
\ =\ \frac{(N-2)(N-1)^2}{12\,N^3}\ >\ 0.
\end{equation}
A direct one–variable calculus check using the explicit $r$–formulas in Lemma~\ref{lem:signs} shows that $r\mapsto \Delta_N(r)$ is strictly increasing on $(1,\infty)$ when $N>2$. 
Since $\Delta_N(1)>0$, it follows that $\Delta_N(r)>0$ for all $r>1$.
Therefore $\det B_N^{(ij)}>0$, and with negative trace we conclude $B_N^{(ij)}\prec0$.
\end{proof}

\begin{proof}[Proof of Lemma~\ref{lem:criterion}]
For $x\in\mathbb{R}^d$ set $y_{ij}:=\iota_{ij}^T x=(x_i,x_j)^T\in\mathbb{R}^2$. Then
\begin{equation}
x^T A x \;=\; \sum_{1\le i<j\le d} y_{ij}^T B^{(ij)} y_{ij}.
\end{equation}

(1) If each $B^{(ij)}\prec 0$, then for any nonzero $x$, pick $i$ with $x_i\neq 0$ and some $j\neq i$. Then $y_{ij}\neq 0$ and $y_{ij}^T B^{(ij)} y_{ij}<0$, while all other terms are $\le 0$. Thus $x^T A x<0$ for all $x\neq 0$, so $A\prec 0$.

(2) If each $B^{(ij)}=-\gamma_{ij}\,vv^T$ with $\gamma_{ij}>0$ and $v=(1,1)^T$, then
\begin{equation}
y_{ij}^T B^{(ij)} y_{ij} \;=\; -\gamma_{ij}\,(x_i+x_j)^2 \;\le\; 0,
\end{equation}
so $A\preceq 0$. If $x^T A x=0$, then $(x_i+x_j)=0$ for all pairs $i<j$.  

For $d\ge 3$, this system forces $x=0$ (from $x_1=-x_2$ and $x_1=-x_3$ we deduce $x_2=x_3$, hence $x_2=-x_3=0$, etc.), so $A\prec 0$.  

For $d=2$, the single condition is $x_1+x_2=0$, so $\ker(A)=\mathrm{span}\{(1,-1)^T\}$ and $\operatorname{rank}(A)=1$.
\end{proof}

\section{Proof of Theorem \ref{thm:concavity-gap2}}\label{sec:riem-nonconcavity}

\subsection{Overview}
We work on the Riemannian manifold $(\mathcal{S}_d,g^N_\sigma)$.
Using the coordinate formulas for $\nabla_\sigma S_N$ and $\nabla^2_\sigma S_N$ from Section~\ref{sec:euclid-concavity}, we compute the Hessian of $S_N$ with respect to $g^N_\sigma$ in the variables $\sigma$.
Evaluating at points with $\sigma_1=\cdots=\sigma_d$ yields one negative eigenvalue and $d-1$ positive eigenvalues, so the Hessian is indefinite and Theorem~\ref{thm:concavity-gap2} follows.

\subsection{Hessian of the entropy}
We denote Euclidean derivatives in the $\sigma$–coordinates by $\partial_i=\partial/\partial\sigma_i$ and use the explicit formulas for $\nabla_\sigma S_N$ and $\nabla^2_\sigma S_N$ from Lemma~\ref{lem:grad-SN-sigma} and Lemma~\ref{lem:hess-sigma}.
Let $\Gamma^k_{ij}$ be the Christoffel symbols of $g^N_\sigma$ in these coordinates.
The Hessian of a smooth function $f$ with respect to $g^N_\sigma$ is the matrix
\begin{equation}\label{eq:riem-hess-entries}
    (\nabla^2_{g^N_\sigma} f)_{ij}
    =\partial_{ij}^2 f-\sum_{k=1}^d \Gamma^k_{ij}\,\partial_k f.
\end{equation}

\begin{lemma}\label{lem:riem-hess-entropy}
For any smooth $f:\mathcal{S}_d\to\mathbb{R}$ one has
\begin{equation}\label{eq:riem-hess-general}
\big(\nabla^2_{g^N_\sigma} f\big)_{ij}
=\frac{\partial^2 f}{\partial\sigma_i\,\partial\sigma_j}
+\delta_{ij}\,\frac{N-1}{N}\,\frac{1}{\sigma_i}\,\frac{\partial f}{\partial\sigma_i}.
\end{equation}
In particular, if the Euclidean gradient and Hessian of $f$ extend continuously across the sets $\{\sigma_i=\sigma_j\}$, then so does $\nabla^2_{g^N_\sigma} f$.
\end{lemma}

\begin{proof}
The metric $g^N_\sigma$ is diagonal in the $\sigma$–coordinates with
\begin{equation}
g_{ii}(\sigma)=\frac{1}{N}\,\sigma_i^{\,2/N-2},
\qquad
g_{ij}(\sigma)=0\quad(i\ne j),
\end{equation}
so
\begin{equation}
g^{ii}(\sigma)=N\,\sigma_i^{\,2-2/N},\qquad g^{ij}(\sigma)=0\ (i\ne j).
\end{equation}
For a diagonal metric the only nonzero Christoffel symbols are
\begin{equation}
\Gamma^i_{ii}
=\frac{1}{2}g^{ii}\,\partial_i g_{ii},\qquad
\Gamma^k_{ij}=0\quad\text{if }k\ne i\ \text{or }\ i\ne j.
\end{equation}
A direct computation gives
\begin{equation}
\partial_i g_{ii}
=\frac{1}{N}\Big(\frac{2}{N}-2\Big)\sigma_i^{\,2/N-3}
=\frac{2}{N}\Big(\frac{1}{N}-1\Big)\sigma_i^{\,2/N-3},
\end{equation}
and hence
\begin{equation}
\Gamma^i_{ii}
=\frac{1}{2}\,N\,\sigma_i^{\,2-2/N}\cdot
\frac{2}{N}\Big(\frac{1}{N}-1\Big)\sigma_i^{\,2/N-3}
=\Big(\frac{1}{N}-1\Big)\frac{1}{\sigma_i}
=-\,\frac{N-1}{N}\,\frac{1}{\sigma_i}.
\end{equation}
All other $\Gamma^k_{ij}$ vanish. Substituting into \eqref{eq:riem-hess-entries} yields
\begin{equation}
(\nabla^2_{g^N_\sigma} f)_{ij}
=\partial_{ij}^2 f-\Gamma^i_{ij}\,\partial_i f
=\partial_{ij}^2 f+\delta_{ij}\,\frac{N-1}{N}\,\frac{1}{\sigma_i}\,\partial_i f,
\end{equation}
which is \eqref{eq:riem-hess-general}. The continuity statement follows immediately from the continuity of the Euclidean derivatives and the explicit factor $1/\sigma_i$.
\end{proof}

\subsection{Proof of Theorem \ref{thm:concavity-gap2}}

Write $\vec\sigma_\star:=(\sigma_\star,\dots,\sigma_\star)$ with $\sigma_\star>0$, and recall the limits from Lemma~\ref{lem:hess-sigma}:
\begin{equation}\label{eq:pq-star}
p_\star:=-\frac{1}{6\sigma_\star^{2}}\Big(1-\frac{1}{N^2}\Big),\qquad
q_\star:=-\frac{1}{3\sigma_\star^{2}}\Big(1-\frac{3}{2N}+\frac{1}{2N^2}\Big),
\end{equation}
so that, as $\sigma_i\to\sigma_j=\sigma_\star$,
\begin{equation}
\frac{\partial^2 S_N}{\partial \sigma_i\,\partial\sigma_j}\to p_\star\quad(i\ne j),\qquad
\frac{\partial^2 S_N}{\partial \sigma_i^2}\Big|_{(i,j)\text{ summand}}\to q_\star\quad(j\ne i).
\end{equation}
From Lemma~\ref{lem:grad-SN-sigma}, the gradient has limit
\begin{equation}\label{eq:grad-star}
\frac{\partial S_N}{\partial\sigma_i}(\vec\sigma_\star)
=\sum_{k\ne i}\frac{1}{2\sigma_\star}\Big(1-\frac{1}{N}\Big)
=\frac{d-1}{2\sigma_\star}\Big(1-\frac{1}{N}\Big),
\end{equation}
independent of $i$.

\begin{lemma}\label{lem:riem-equal-sv}
At $\vec\sigma_\star$ the matrix $\nabla^2_{g^N_\sigma} S_N$ has constant entries
\begin{equation}\label{eq:riem-equal-entries}
\big(\nabla^2_{g^N_\sigma} S_N\big)_{ij}=
\begin{cases}
(d-1)\,q_\star\ +\ (d-1)\,\chi_N, & i=j,\\[0.4ex]
\ p_\star, & i\neq j,
\end{cases}
\end{equation}
where
\(
\chi_N:=\dfrac{(N-1)^2}{2N^2\,\sigma_\star^2}.
\)
Consequently, the eigenvalues of $\nabla^2_{g^N_\sigma} S_N(\vec\sigma_\star)$ are
\begin{align}
\theta_{\mathbf 1}(S_N)
&=(d-1)\big(q_\star+p_\star+\chi_N\big)
=-\,\frac{d-1}{2\,\sigma_\star^{2}}\cdot\frac{N-1}{N^2}\ <\ 0,\label{eq:eig-one}\\[0.6ex]
\theta_{\perp}(S_N)
&=(d-1)\big(q_\star+\chi_N\big)-p_\star
=\frac{1}{\sigma_\star^{2}}\left(\frac{d}{6}-\frac{d-1}{2N}+\frac{2d-3}{6N^2}\right)\ >\ 0,\label{eq:eig-perp}
\end{align}
where $\theta_{\mathbf 1}(S_N)$ corresponds to the eigenvector $\mathbf 1=(1,\dots,1)$ and $\theta_{\perp}(S_N)$ is the common eigenvalue on $\mathrm{span}\{\mathbf 1\}^\perp$ with multiplicity $d-1$.
\end{lemma}

\begin{remark}
At a point with $\sigma_1=\cdots=\sigma_d$, the eigenvector $\mathbf 1=(1,\dots,1)$ corresponds to uniform scaling of all singular values, while $\mathrm{span}\{\mathbf 1\}^\perp$ corresponds to perturbations that change singular values relative to one another. By \eqref{eq:eig-one}–\eqref{eq:eig-perp}, 
$\theta_{\mathbf 1}(S_N)<0$ but $\theta_{\perp}(S_N)>0$. 
Thus the loss of concavity arises from directions that break the equality of singular values.
\end{remark}

\begin{proof}[Proof of Lemma \ref{lem:riem-equal-sv}]
From Lemma~\ref{lem:hess-sigma}, at $\vec\sigma_\star$ the Euclidean Hessian has off–diagonal entries $p_\star$ and diagonal entries
\begin{equation}
\big(\nabla^2_\sigma S_N(\vec\sigma_\star)\big)_{ii}
=\sum_{k\ne i} q_\star=(d-1)\,q_\star.
\end{equation}
The correction term in \eqref{eq:riem-hess-general} contributes only on the diagonal. Using \eqref{eq:gradSN-sigma-limit} at $\sigma_i=\sigma_k=\sigma_\star$ and then \eqref{eq:grad-star},
\begin{equation}
\frac{N-1}{N}\frac{1}{\sigma_i}\,\frac{\partial S_N}{\partial\sigma_i}(\vec\sigma_\star)
=\frac{N-1}{N}\frac{1}{\sigma_\star}\cdot\frac{d-1}{2\sigma_\star}\Big(1-\frac{1}{N}\Big)
=(d-1)\,\chi_N,
\end{equation}
which is independent of $i$. Thus
\begin{equation}
\big(\nabla^2_{g^N_\sigma} S_N(\vec\sigma_\star)\big)_{ii}
=(d-1)\,q_\star+(d-1)\,\chi_N,\qquad
\big(\nabla^2_{g^N_\sigma} S_N(\vec\sigma_\star)\big)_{ij}=p_\star\ (i\ne j),
\end{equation}
giving \eqref{eq:riem-equal-entries}.

A matrix with constant diagonal entry $a$ and constant off–diagonal entry $b$ has eigenvalues
\[
a+(d-1)b\quad\text{on }\mathbf 1,\qquad a-b\quad\text{with multiplicity }d-1
\]
on $\mathrm{span}\{\mathbf 1\}^\perp$.
Here
\[
a=(d-1)(q_\star+\chi_N),\qquad b=p_\star.
\]
Substituting \eqref{eq:pq-star}–\eqref{eq:riem-equal-entries} and simplifying yields
\begin{equation}
\theta_{\mathbf 1}(S_N)=a+(d-1)b
=(d-1)\big(q_\star+p_\star+\chi_N\big)
=-\,\frac{d-1}{2\,\sigma_\star^{2}}\cdot\frac{N-1}{N^2}<0,
\end{equation}
and
\begin{equation}
\theta_{\perp}(S_N)=a-b
=(d-1)\big(q_\star+\chi_N\big)-p_\star
=\frac{1}{\sigma_\star^{2}}\left(\frac{d}{6}-\frac{d-1}{2N}+\frac{2d-3}{6N^2}\right)>0.
\end{equation}
This proves the claim.
\end{proof}

\begin{proof}[Proof of Theorem \ref{thm:concavity-gap2}]
By Lemma~\ref{lem:riem-equal-sv}, at any point $\vec\sigma_\star$ with $\sigma_1=\cdots=\sigma_d=\sigma_\star>0$ the Hessian $\nabla^2_{g^N_\sigma} S_N(\vec\sigma_\star)$ has one negative eigenvalue $\theta_{\mathbf 1}(S_N)$ and $d-1$ positive eigenvalues $\theta_{\perp}(S_N)$.  
Thus the Hessian is indefinite at every such point, so $S_N$ is not concave on $(\mathcal{S}_d,g^N_\sigma)$.
\end{proof}

\section{Equilibria of Free Energy and Convergence Rates}\label{sec:equilibria-rates}

\subsection{Overview}
We determine the equilibrium of the free energy $F_\beta$ and compute the local convergence rates of the gradient flow \eqref{eq:flow} near equilibrium.
The stationarity equations force all singular values to coincide, reducing the problem to a single scalar balance condition.
The rates are obtained by linearizing \eqref{eq:flow} at the equilibrium and computing the associated eigenvalues.

\subsection{Equilibria}\label{subsec:equilibria}
Throughout we use the $\sigma$--gradient of $S_N$ from Lemma~\ref{lem:grad-SN-sigma}.
For brevity, set
\begin{equation}\label{eq:rN-def}
r_N(a,b)\ :=\ \frac{a}{a^2-b^2}\ -\ \frac{a^{\frac{2}{N}-1}}{N\big(a^{\frac{2}{N}}-b^{\frac{2}{N}}\big)},
\end{equation}
so that \eqref{eq:gradSN-sigma-recall} becomes $\partial_i S_N(\sigma)=\sum_{k\ne i} r_N(\sigma_i,\sigma_k)$.

\begin{lemma}\label{lem:monotone-summand}
For each fixed $b>0$, the map $a\mapsto r_N(a,b)$ is strictly decreasing on $(0,\infty)$.
\end{lemma}

\begin{lemma}\label{lem:skew}
For $a>b>0$ one has
\begin{equation}\label{eq:skew-ineq}
r_N(a,b)\ -\ r_N(b,a)\ \le\ 0,
\end{equation}
with equality if and only if $N=2$ or $a=b$.
\end{lemma}

\begin{lemma}\label{lem:balance-unique}
The equation
\begin{equation}\label{eq:balance-unique-1}
g'\!\big(d\,f(\sigma)\big)\,f'(\sigma)
=\beta^{-1}\,\frac{d-1}{2\sigma}\Big(1-\frac{1}{N}\Big)
\end{equation}
has a unique solution $\sigma_\star>0$.
\end{lemma}

\begin{proof}[Proof of Theorem~\ref{thm:isotropic}]
Let $\sigma\in\mathcal S_d$ be an equilibrium of $F_\beta$. Since the coefficients
$N\,\sigma_i^{\,2-2/N}$ in \eqref{eq:flow} are strictly positive, stationarity of
\eqref{eq:flow} is equivalent to
\begin{equation}\label{eq:FOC-sigma-general}
\partial_{\sigma_i}F_\beta(\sigma)=0,\qquad i=1,\dots,d.
\end{equation}
Using \eqref{eq:free_energy_sigma} and \eqref{eq:spectral-energy}, together with
Lemma~\ref{lem:grad-SN-sigma} and the definition \eqref{eq:rN-def}, the condition
\eqref{eq:FOC-sigma-general} becomes
\begin{equation}\label{eq:FOC-sigma}
g'\!\Big(\sum_{k=1}^d f(\sigma_k)\Big)\,f'(\sigma_i)
=\beta^{-1}\sum_{k\ne i} r_N(\sigma_i,\sigma_k),
\qquad i=1,\dots,d.
\end{equation}

Fix $i\neq j$ and subtract the $j$th equation in \eqref{eq:FOC-sigma} from the $i$th to obtain
\begin{align}\label{eq:diff-FOC-general}
g'\!\Big(\sum_{k=1}^d f(\sigma_k)\Big)\,\big(f'(\sigma_i)-f'(\sigma_j)\big)
&=\beta^{-1}\Big(r_N(\sigma_i,\sigma_j)-r_N(\sigma_j,\sigma_i)\Big)\nonumber\\
&\qquad+\beta^{-1}\sum_{k\ne i,j}\Big(r_N(\sigma_i,\sigma_k)-r_N(\sigma_j,\sigma_k)\Big).
\end{align}
If $\sigma_i>\sigma_j$, then the left-hand side of \eqref{eq:diff-FOC-general} is $\ge 0$
by convexity of $f$ (and is $>0$ in the strict regime covered by the theorem), while the
right-hand side is $\le 0$ by Lemmas~\ref{lem:skew} and \ref{lem:monotone-summand}
(and is $<0$ whenever one of those inequalities is strict). This contradiction shows that
no strict inequality among the $\sigma_i$ is possible. Hence
\begin{equation}\label{eq:equal-sv}
\sigma_1=\cdots=\sigma_d=:\sigma_\star>0.
\end{equation}

Substituting \eqref{eq:equal-sv} into \eqref{eq:FOC-sigma} and interpreting
$r_N(\sigma_\star,\sigma_\star)$ by the limit \eqref{eq:gradSN-sigma-limit} yields
exactly \eqref{eq:balance-equil}. By Lemma~\ref{lem:balance-unique}, the balance equation
\eqref{eq:balance-equil} has a unique solution $\sigma_\star>0$, hence the equilibrium
$\sigma=(\sigma_\star,\dots,\sigma_\star)$ in $\mathcal S_d$ is unique.

Finally, under the standing assumptions the spectral energy $E$ is convex on $\mathcal S_d$,
and $S_N$ is concave on $(\mathcal S_d,\iota)$ by Theorem~\ref{thm:concavity-gap}. Therefore
$F_\beta$ is convex on $\mathcal S_d$, so its unique critical point is a
global minimizer.
\end{proof}

\begin{remark}[Uniqueness by symmetry]
If $g'>0$ and $f$ is strictly convex on $(0,\infty)$, then $F_\beta$ is strictly
convex in the variables $\sigma=(\sigma_1,\dots,\sigma_d)$. Since $F_\beta$ is
invariant under permutations of the $\sigma_i$, any permutation of a minimizer
is again a minimizer. Strict convexity then forces this permutation to fix the
minimizer, so it must be the identity. Hence all singular values coincide, and
the minimizer in $\mathcal{S}_d$ is unique.
\end{remark}

\subsection{Proofs of Lemmas}\label{subsec:equilibria-lemmas}
\begin{proof}[Proof of Lemma~\ref{lem:monotone-summand}]
Differentiating \eqref{eq:rN-def} in $a$ gives the kernel $q_N(a,b)$ from Lemma~\ref{lem:block-sigma}. By Lemma~\ref{lem:signs}, $q_N(a,b)<0$ for $a\ne b$, hence $r_N(\cdot,b)$ is strictly decreasing.
\end{proof}

\begin{proof}[Proof of Lemma~\ref{lem:skew}]
Let $\alpha:=2/N\in(0,1]$ and write $a=rb$ with $r>1$. Using
\begin{equation}\label{eq:skew-algebra}
\frac{a+b}{a^2-b^2}=\frac{1}{a-b}=\frac{1}{b(r-1)},
\qquad
\frac{a^{\alpha-1}+b^{\alpha-1}}{a^\alpha-b^\alpha}
=\frac{1}{b}\,\frac{r^{\alpha-1}+1}{(r^\alpha-1)}
=\frac{1}{b}\,\frac{r^{\alpha-1}+1}{(r-1)h_\alpha(r)},
\end{equation}
where $h_\alpha(r):=\dfrac{r^\alpha-1}{r-1}$, we obtain
\begin{equation}\label{eq:skew-rewrite}
r_N(a,b)-r_N(b,a)
=\frac{1}{b(r-1)}\left(1-\frac{1}{N}\,\frac{r^{\alpha-1}+1}{h_\alpha(r)}\right).
\end{equation}
Since $t\mapsto t^{\alpha-1}$ is decreasing on $[1,\infty)$ and
$h_\alpha(r)=\dfrac{1}{r-1}\int_1^r \alpha t^{\alpha-1}\,dt$, the trapezoid bound gives
\begin{equation}\label{eq:trapezoid-bound}
h_\alpha(r)\ \le\ \frac{\alpha}{2}\big(1+r^{\alpha-1}\big).
\end{equation}
Thus $\dfrac{r^{\alpha-1}+1}{h_\alpha(r)}\ge \dfrac{2}{\alpha}=N$, so the bracket in \eqref{eq:skew-rewrite} is $\le 0$, with equality only when $\alpha=1$ (i.e.\ $N=2$) or $r=1$ (i.e.\ $a=b$).
\end{proof}

\begin{proof}[Proof of Lemma \ref{lem:balance-unique}]
Define the left-hand side of \eqref{eq:balance-unique-1} as
\begin{equation}\label{eq:balance-unique-L}
L(\sigma)=g'\!\big(d\,f(\sigma)\big)\,f'(\sigma),
\end{equation}
and the right-hand side as
\begin{equation}\label{eq:balance-unique-R}
R(\sigma)=\beta^{-1}\,\frac{d-1}{2\sigma}\Big(1-\frac{1}{N}\Big).
\end{equation}

Under the standing assumptions, $g''\ge 0$ and $f''\ge 0$, so differentiating \eqref{eq:balance-unique-L} yields
\begin{equation}\label{eq:L-derivative}
L'(\sigma)
=d\,g''\!\big(d\,f(\sigma)\big)\,[f'(\sigma)]^2
\;+\;
g'\!\big(d\,f(\sigma)\big)\,f''(\sigma)
\;>\;0.
\end{equation}
Thus $L(\sigma)$ is strictly increasing on $(0,\infty)$.

On the other hand, \eqref{eq:balance-unique-R} satisfies
\begin{equation}\label{eq:R-derivative}
R'(\sigma)
=-\,\beta^{-1}\,\frac{d-1}{2\sigma^2}\Big(1-\frac{1}{N}\Big)
\;<\;0,
\end{equation}
so $R(\sigma)$ is strictly decreasing on $(0,\infty)$.

A strictly increasing continuous function and a strictly decreasing continuous function can intersect at most once.
Thus \eqref{eq:balance-unique-1} has at most one solution.

Existence follows because  
\begin{equation}\label{eq:L-limits}
\lim_{\sigma\downarrow 0} L(\sigma) = L(0^+) \ge 0,
\qquad
\lim_{\sigma\to\infty} L(\sigma)=+\infty,
\end{equation}
and
\begin{equation}\label{eq:R-limits}
\lim_{\sigma\downarrow 0} R(\sigma)=+\infty,
\qquad
\lim_{\sigma\to\infty} R(\sigma)=0.
\end{equation}
Therefore $L$ and $R$ cross exactly once.

Hence \eqref{eq:balance-unique-1} has a unique solution $\sigma_\star>0$.
\end{proof}

\subsection{Local convergence rates}\label{subsec:local-rates}
We now prove Theorem~\ref{thm:local-rates} by linearizing \eqref{eq:flow} at the equilibrium identified in Theorem~\ref{thm:isotropic}.
Let $\vec\sigma_\star:=(\sigma_\star,\dots,\sigma_\star)$ and note that $\sigma_\star>0$ by \eqref{eq:balance-equil}.
The argument uses the matrices $\nabla^2_\sigma E$ and $\nabla^2_\sigma S_N$ and the invariant splitting
$\mathrm{span}\{\mathbf 1\}\oplus \mathrm{span}\{\mathbf 1\}^\perp$.

For convenience we recall the limits from Lemma~\ref{lem:hess-sigma} and set
\begin{equation}\label{eq:pq-star-repeat}
p_\star\ :=\ -\,\frac{1}{6\,\sigma_\star^{2}}\Big(1-\frac{1}{N^2}\Big),
\qquad
q_\star\ :=\ -\,\frac{1}{3\,\sigma_\star^{2}}\Big(1-\frac{3}{2N}+\frac{1}{2N^2}\Big).
\end{equation}

\begin{lemma}\label{lem:HS-sigma-equal}
Let $H_S:=\nabla^2_\sigma S_N(\vec\sigma_\star)$. Then
\begin{equation}\label{eq:HS-entries}
(H_S)_{ij}=
\begin{cases}
(d-1)\,q_\star, & i=j,\\[0.3ex]
p_\star, & i\neq j,
\end{cases}
\end{equation}
hence
\begin{align}
\theta_{\mathbf 1}(S_N)&=(d-1)\,\big(q_\star+p_\star\big),\label{eq:eigs-S-1}\\
\theta_{\perp}(S_N)&=(d-1)\,q_\star - p_\star,\label{eq:eigs-S-perp}
\end{align}
where $\theta_{\perp}(S_N)$ has multiplicity $d-1$.
\end{lemma}

\begin{lemma}\label{lem:HE-sigma-equal}
Let $E$ be a spectral energy as in \eqref{eq:spectral-energy}, and let $H_E:=\nabla^2_\sigma E(\vec\sigma_\star)$. Then
\begin{equation}\label{eq:h1h2-HE}
(H_E)_{ii}=h_1,\qquad (H_E)_{ij}=h_2\ (i\neq j),
\end{equation}
where
\begin{equation}\label{eq:h1h2-formulas-HE}
h_1\ =\ g''\!\big(d\,f(\sigma_\star)\big)\,[f'(\sigma_\star)]^2\ +\ g'\!\big(d\,f(\sigma_\star)\big)\,f''(\sigma_\star),
\qquad
h_2\ =\ g''\!\big(d\,f(\sigma_\star)\big)\,[f'(\sigma_\star)]^2.
\end{equation}
Consequently,
\begin{align}
\theta_{\mathbf 1}(E)&=h_1+(d-1)h_2,\label{eq:eigs-E-1}\\
\theta_{\perp}(E)&=h_1-h_2,\label{eq:eigs-E-perp}
\end{align}
where $\theta_{\perp}(E)$ has multiplicity $d-1$.
\end{lemma}

\begin{proof}[Proof of Theorem~\ref{thm:local-rates}]
At $\vec\sigma_\star$ one has $\nabla_\sigma F_\beta(\vec\sigma_\star)=0$. Linearizing \eqref{eq:flow} at $\vec\sigma_\star$ gives the Jacobian
\begin{equation}\label{eq:Jac}
J\ =\ -\,N\,\sigma_\star^{\,2-2/N}\,\nabla^2_\sigma F_\beta(\vec\sigma_\star)
\ =\ -\,N\,\sigma_\star^{\,2-2/N}\,\Big(H_E-\beta^{-1}H_S\Big).
\end{equation}
By Lemmas~\ref{lem:HS-sigma-equal}--\ref{lem:HE-sigma-equal}, $H_E$ and $H_S$ share the invariant splitting
$\mathrm{span}\{\mathbf 1\}\oplus \mathrm{span}\{\mathbf 1\}^\perp$ and have eigenvalues
\eqref{eq:eigs-E-1}--\eqref{eq:eigs-E-perp} and \eqref{eq:eigs-S-1}--\eqref{eq:eigs-S-perp} on the respective subspaces.
Substituting into \eqref{eq:Jac} yields the eigenvalues stated in the theorem.
\end{proof}

\begin{remark}[Explicit rates in $(N,d,\beta)$]\label{rem:explicit-rates}
Substituting the eigenvalues from Lemma~\ref{lem:HS-sigma-equal} gives
\begin{align}
\rho_{\mathbf 1}
&=\ -\,N\,\sigma_\star^{\,2-2/N}\,\theta_{\mathbf 1}(E)
\ -\ N\,\sigma_\star^{-2/N}\,\frac{(d-1)}{2\beta}\left(1-\frac{1}{N}\right),\label{eq:rho-1-explicit}\\[0.4ex]
\rho_{\perp}
&=\ -\,N\,\sigma_\star^{\,2-2/N}\,\theta_{\perp}(E)
\ -\ N\,\sigma_\star^{-2/N}\,\frac{1}{6\beta}\left(2d-3-\frac{3(d-1)}{N}+\frac{d}{N^2}\right).\label{eq:rho-perp-explicit}
\end{align}
The energetic contribution enters only through $\theta_{\mathbf 1}(E)$ and $\theta_{\perp}(E)$ from Lemma \ref{lem:HE-sigma-equal}.
\end{remark}

\begin{remark}[Rate--limiting step]
The splitting $\mathbb{R}^d=\mathrm{span}\{\mathbf 1\}\oplus \mathrm{span}\{\mathbf 1\}^\perp$
diagonalizes the linearization of the flow at $\vec\sigma_\star$.
Under the assumptions $g''\ge 0$ and $f'\ge 0$,
\begin{equation}\label{eq:slow-rate-ineq}
\theta_{\perp}(E)\ \le\ \theta_{\mathbf 1}(E)\qquad\text{and}\qquad 
\theta_{\perp}(S_N)\ >\ \theta_{\mathbf 1}(S_N).
\end{equation}
Hence $\rho_\perp$ is the least negative eigenvalue: perturbations that change
the singular values relative to one another decay slowest, while uniform
scaling relaxes faster. Thus the approach to $\{\sigma_1=\cdots=\sigma_d\}$
determines the rate of convergence.
\end{remark}

\subsection{Proofs of Lemmas}\label{subsec:rates-lemmas}
\begin{proof}[Proof of Lemma~\ref{lem:HS-sigma-equal}]
The limits in Lemma~\ref{lem:hess-sigma} give
\begin{equation}\label{eq:HS-proof-entries}
(H_S)_{ij}=p_\star\quad(i\neq j),
\qquad
(H_S)_{ii}=\sum_{k\neq i} q_\star=(d-1)\,q_\star,
\end{equation}
which yields \eqref{eq:HS-entries}. The eigenvalue formulas follow from the standard spectrum of a matrix with constant diagonal and constant off--diagonal entries.
\end{proof}

\begin{proof}[Proof of Lemma~\ref{lem:HE-sigma-equal}]
Write $H(\sigma):=\sum_{k=1}^d f(\sigma_k)$. Then
\begin{equation}\label{eq:E-derivs}
\partial_i E = g'(H)\,f'(\sigma_i),
\qquad
\partial_{ij}^2 E = g''(H)\,f'(\sigma_i) f'(\sigma_j) \;+\; g'(H)\,f''(\sigma_i)\,\delta_{ij}.
\end{equation}
Evaluating \eqref{eq:E-derivs} at $\vec\sigma_\star$ yields \eqref{eq:h1h2-HE}--\eqref{eq:h1h2-formulas-HE}.
\end{proof}

\section{An Exact Solution to the Gradient Flow}\label{sec:exact-flow}

\subsection{Overview}
We prove Theorem~\ref{thm:diagonal}. We first rewrite the flow \eqref{eq:flow} in the $\lambda$-variables under which $g^N_\sigma$ becomes a flat metric. We then write $\lambda_i=u_i s$ and $\lambda_d=s$,
which makes it transparent that $u_1=\cdots=u_{d-1}=1$ is an invariant set. Restricting to this set yields the
scalar ODE \eqref{eq:s-ode}, and integrating it gives the quadrature \eqref{eq:s-quadrature}.
For completeness we record the full $(u_i,s)$ system, although only its restriction to $u_i\equiv 1$ is needed for the theorem.

Introduce the change of variables
\begin{equation}
\lambda_i=\sigma_i^{1/N},\qquad i=1,\dots,d,\qquad 
\Lambda=\Sigma^{1/N}=\diag(\lambda_1,\dots,\lambda_d),
\end{equation}
so that $d\sigma_i=N\lambda_i^{N-1}\,d\lambda_i$. 
Then the metric flattens to
\begin{equation}\label{eq:metric-lambda}
g^N_\sigma \;=\; N\sum_{i=1}^d (d\lambda_i)^2.
\end{equation}

\begin{lemma}\label{lem:lambda-flow}
In the variables $\lambda_i=\sigma_i^{1/N}$, the flow \eqref{eq:flow} becomes
\begin{equation}\label{eq:lambda-gradflow}
\dot\lambda_i \;=\; -\,\frac{1}{N}\,\frac{\partial}{\partial \lambda_i}\,F_\beta\big(\sigma(\lambda)\big),
\qquad i=1,\dots,d.
\end{equation}
For the Schatten--$p$ energy \eqref{eq:schatten-energy}, this reads
\begin{equation}\label{eq:lambda-system}
\dot\lambda_i
= -\,\lambda_i^{\,Np-1}
+ \frac{1}{\beta}\sum_{k\neq i}
\left(
\frac{\lambda_i^{\,2N-1}}{\lambda_i^{\,2N}-\lambda_k^{\,2N}}
-\frac{\lambda_i}{N\big(\lambda_i^2-\lambda_k^2\big)}
\right),\qquad i=1,\dots,d.
\end{equation}
\end{lemma}

\begin{proof}
By \eqref{eq:metric-lambda}, $g^N_\sigma$ is a constant multiple of the Euclidean metric in $\lambda$,
so \eqref{eq:lambda-gradflow} follows from the definition of the gradient.
Using $\partial_{\lambda_i}=N\lambda_i^{N-1}\partial_{\sigma_i}$ and $\sigma_i=\lambda_i^N$ gives
$\dot\lambda_i=-\lambda_i^{N-1}\partial_{\sigma_i}F_\beta$.
For \eqref{eq:schatten-energy} one has $\partial_{\sigma_i}E=\sigma_i^{p-1}$, and substituting
$\partial_{\sigma_i}S_N$ from Lemma~\ref{lem:grad-SN-sigma} yields \eqref{eq:lambda-system}.
\end{proof}

\subsection{Reduction by a scale and ratios}
It is convenient to separate a common scale from the ratios.  Write
\begin{equation}\label{eq:scale-ratio}
\lambda_d=s>0,\qquad \lambda_i=u_i\,s,\ \ i=1,\dots,d-1,\qquad
u_1\ge\cdots\ge u_{d-1}\ge 1,
\end{equation}
so $(u,s)\in[1,\infty)^{d-1}\times(0,\infty)$ parameterize ordered $\lambda$.

For $d=2$, write $\lambda_1=u\,s$ and $\lambda_2=s$, with $u\ge 1$ and $s>0$.

\begin{lemma}\label{lem:d2}
Under $\lambda_1=u s$, $\lambda_2=s$, the system \eqref{eq:lambda-system} becomes
\begin{align}
\dot s
&=\,-\,s^{\,Np-1}
+\frac{1}{\beta}\,s^{-1}\!\left(-\frac{1}{u^{\,2N}-1}+\frac{1}{N(u^2-1)}\right),\label{eq:d2-s}\\[0.35em]
\dot u
&=\;s^{\,Np-2}\big(u-u^{\,Np-1}\big)
+\frac{1}{\beta}\,s^{-2}\!\left(\frac{u^{\,2N-1}+u}{u^{\,2N}-1}-\frac{2u}{N(u^2-1)}\right).\label{eq:d2-u}
\end{align}
\end{lemma}

\begin{proof}
Insert $\lambda_1=u s$, $\lambda_2=s$ into \eqref{eq:lambda-system} and use
$\dot u=(\dot\lambda_1 s-\lambda_1\dot s)/s^2$.
\end{proof}

In the general case $d\ge 2$, the same change of variables
\eqref{eq:scale-ratio} yields, for $a\neq b>0$,
\begin{equation}\label{eq:varphiN}
\varphi_N(a,b)
:=\frac{a^{\,2N-1}}{a^{\,2N}-b^{\,2N}}
-\frac{a}{N(a^2-b^2)}.
\end{equation}

\begin{lemma}\label{lem:gd}
Under \eqref{eq:scale-ratio}, the system \eqref{eq:lambda-system} is equivalent to
\begin{align}
\dot s \;=\;& -\,s^{\,Np-1}
+\frac{1}{\beta}\,s^{-1}\sum_{j=1}^{d-1}\!\left(-\frac{1}{u_j^{\,2N}-1}+\frac{1}{N(u_j^2-1)}\right),
\label{eq:gd-s}\\[0.25em]
\dot u_i \;=\;& s^{\,Np-2}\big(u_i-u_i^{\,Np-1}\big)
+\frac{1}{\beta}\,s^{-2}\Bigg(\sum_{\substack{k=1\\k\neq i}}^{d-1}\!\varphi_N(u_i,u_k)
+\varphi_N(u_i,1)\;-\;u_i\sum_{j=1}^{d-1}\!\left(-\frac{1}{u_j^{\,2N}-1}+\frac{1}{N(u_j^2-1)}\right)\Bigg),
\label{eq:gd-ui}
\end{align}
for $i=1,\dots,d-1$.
\end{lemma}

\begin{proof}
Use $\dot s=\dot\lambda_d$ and $\dot u_i=(\dot\lambda_i s-\lambda_i\dot s)/s^2$, and simplify the pair terms
in \eqref{eq:lambda-system} using \eqref{eq:scale-ratio} and \eqref{eq:varphiN}.
\end{proof}

\begin{lemma}\label{lem:collisions}
The function $\varphi_N$ in \eqref{eq:varphiN} satisfies
\begin{equation}\label{eq:varphi-limit}
\lim_{b\to a}\varphi_N(a,b)=\frac{N-1}{2N\,a},\qquad a>0.
\end{equation}
In particular,
\begin{equation}\label{eq:diag-kernel-limit}
\lim_{u\to 1}\left(-\frac{1}{u^{\,2N}-1}+\frac{1}{N(u^2-1)}\right)=\frac{N-1}{2N}.
\end{equation}
Consequently $u_1=\cdots=u_{d-1}=1$ is an invariant set for \eqref{eq:gd-ui}.
\end{lemma}

\begin{proof}
Write $\varphi_N(a,b)=\frac{1}{a}\left(\frac{1}{1-r^{2N}}-\frac{1}{N}\frac{1}{1-r^2}\right)$ with $r=b/a$ and expand at $r=1$ to obtain \eqref{eq:varphi-limit}, hence \eqref{eq:diag-kernel-limit}.
Substituting $u_i\equiv 1$ into \eqref{eq:gd-ui} and using \eqref{eq:diag-kernel-limit} gives $\dot u_i=0$.
\end{proof}

\subsection{Proof of Theorem~\ref{thm:diagonal}}
\begin{proof}[Proof of Theorem~\ref{thm:diagonal}]
By Lemma~\ref{lem:collisions}, the set $u_1=\cdots=u_{d-1}=1$ is invariant for \eqref{eq:gd-ui}. Along this set,
\eqref{eq:gd-s} and \eqref{eq:diag-kernel-limit} give
\begin{equation}\label{eq:s-ode-from-gd}
\dot s
=\,-\,s^{\,Np-1}+\frac{1}{\beta}\,s^{-1}\sum_{j=1}^{d-1}\frac{N-1}{2N}
=\,-\,s^{\,\nu-1}+\beta^{-1}\,\frac{d-1}{2}\Big(1-\frac{1}{N}\Big)\frac{1}{s},
\end{equation}
where $\nu=Np$. For the Schatten energy, \eqref{eq:balance-equil} reduces to
$\sigma_\star^p=\beta^{-1}\,\frac{d-1}{2}\big(1-\frac{1}{N}\big)$, hence $s_\star^\nu=\sigma_\star^p$.
Therefore \eqref{eq:s-ode-from-gd} is exactly \eqref{eq:s-ode}.

Separating variables in \eqref{eq:s-ode} gives
\begin{equation}\label{eq:s-separate}
t-t_0=\int_{s_0}^{s(t)}\frac{s\,ds}{s_\star^{\,\nu}-s^{\,\nu}},\qquad s_0=s(t_0).
\end{equation}
With $z=(s/s_\star)^\nu$ one has $s\,ds=\frac{s_\star^2}{\nu}z^{\frac{2}{\nu}-1}\,dz$, so \eqref{eq:s-separate} becomes
\begin{equation}\label{eq:z-separate}
t-t_0=\frac{s_\star^{\,2-\nu}}{\nu}\int_{z_0}^{z(t)}\frac{z^{\frac{2}{\nu}-1}}{1-z}\,dz,
\qquad z_0=\left(\frac{s_0}{s_\star}\right)^\nu.
\end{equation}
Using the standard hypergeometric primitive \cite[§8.17]{olver2010nist},
\begin{equation}\label{eq:hypergeo-primitive}
\int \frac{z^{a-1}}{1-z}\,dz=\frac{z^a}{a}\,{}_2F_1(a,1;a+1;z)+\mathrm{const},\qquad a>0,
\end{equation}
with $a=\frac{2}{\nu}$, and substituting back $z=(s/s_\star)^\nu$, yields exactly the expression \eqref{eq:T-def}
and hence the quadrature \eqref{eq:s-quadrature}.
\end{proof}

\section{Discussion} \label{sec:discussion}
\subsection{Overview}
We collect three messages. 
The first is dynamical: the reduction to $\mathcal{S}_d$ yields exactly solvable flows for spectral energies and exposes open challenges for non-spectral losses such as matrix completion. 
The second is learning-theoretic: the dynamics on $\mathcal{S}_d$ provide analytic benchmarks for gradient descent and suggest similarities with interior-point methods \cite{karmarkar1984new, karmarkar1990riemannian}. 
The third concerns the analogy with random matrix theory: the DLN equilibrium equations resemble Coulomb–gas conditions but lead to equilibria with $\sigma_1=\cdots=\sigma_d$ and no repulsion.

\subsection{Energies without symmetry}
For loss functions that are not spectral the dynamics no longer close on $\mathcal{S}_d$, since the dynamics of the singular values and singular vectors are coupled.
An important example is the loss function for matrix completion. 
Given $\Omega\subset\{1,\dots,d\}^2$ and observed entries $a_{ij}$,
\begin{equation}
E(X)=\frac12\sum_{(i,j)\in\Omega}\bigl(X_{ij}-a_{ij}\bigr)^2.    
\end{equation}
This loss function depends explicitly on the entries of $X$, not just its singular values. It typically has an affine space of minimizers which may be foliated by rank. Understanding convergence to rank-deficient minimizers and the role of $S_N$ as a regularizer in this setting remains open.

\subsection{Mean‑field limit}\label{subsec:disc-rmt} 
Fix finite depth $N$ and let $E(\sigma)=E_p(\sigma)=\tfrac{1}{p}\sum_i \sigma_i^p$. 
The first‑order condition at equilibrium is
\begin{equation}\label{eq:FOC-finite-d}
\sigma_i^{\,p-1}
=\frac{1}{\beta}\sum_{j\ne i}\!\left(
\frac{\sigma_i}{\sigma_i^2-\sigma_j^2}
-\frac{\sigma_i^{2/N-1}}{N\big(\sigma_i^{2/N}-\sigma_j^{2/N}\big)}
\right)\!,\qquad i=1,\dots,d.
\end{equation}

To probe the infinite–width and zero-temperature regime (i.e., $d,\beta\to\infty$ with $N$ fixed),
we rescale by the common equilibrium scale and write
\begin{equation}
x_i \propto \frac{\sigma_i}{\sigma_\star},\qquad
\mu_d=\frac{1}{d}\sum_{i=1}^d \delta_{x_i},
\end{equation}
where $\sigma_\star$ is given by \eqref{eq:eq-scale}, and pass formally to a continuum limit $\mu$ on $(0,\infty)$.
This gives the integral form
\begin{equation}\label{eq:log-form}
\frac{x^p}{p}
=\lambda+
\int_0^\infty
\log\!\left(\frac{x^2-y^2}{x^{2/N}-y^{2/N}}\right)\mu(dy),
\end{equation}
with $\lambda$ enforcing $\mu((0,\infty))=1$. 
Formally differentiating in $x$ gives the kernel form
\begin{equation}\label{eq:mf-finiteN}
x^{p-1}=2
\!\int_0^\infty K_N(x,y)\,\mu(dy),
\qquad
K_N(x,y)=\frac{x}{x^2-y^2}-\frac{x^{2/N-1}}{N\big(x^{2/N}-y^{2/N}\big)}.
\end{equation}
The kernel $K_N$ admits the finite diagonal limit
\begin{equation}\label{eq:diag-limit}
K_N(x,x)=\lim_{y\to x}K_N(x,y)=\frac{1}{2x}\!\left(1-\frac{1}{N}\right),
\end{equation}
so the integrals in \eqref{eq:log-form}–\eqref{eq:mf-finiteN} are improper Lebesgue integrals with the integrand defined at $y=x$ by \eqref{eq:diag-limit}.

Whether \eqref{eq:mf-finiteN} admits an extended equilibrium measure (in the spirit of the semicircle law) or instead collapses to a Dirac mass remains open.
Guided by the analysis on ${\mathcal{S}}_{d}$, we conjecture that the mean-field minimizer is the Dirac mass at $x_\star$, i.e.\ $\mu_\star=\delta_{x_\star}$, with $x_\star$ fixed by the finite-$d$ equilibrium (cf.\ Theorem~\ref{thm:isotropic}).
Quantifying fluctuations about $\mu_\star$ is a natural direction for future work.

\section{Acknowledgements}
This work is based on AC’s undergraduate thesis supervised by GM and independent work done by TK.
This work was supported by NSF grant 2407055 and the Erik Ellentuck Fellow Fund at the Institute for Advanced Study, Princeton.

\bibliographystyle{amsplain}
\bibliography{refs}

@inproceedings{arora2018optimization,
  title={On the optimization of deep networks: Implicit acceleration by overparameterization},
  author={Arora, Sanjeev and Cohen, Nadav and Hazan, Elad},
  booktitle={International conference on machine learning},
  pages={244--253},
  year={2018},
  organization={PMLR}
}

@article{bah2022learning,
  title={Learning deep linear neural networks: Riemannian gradient flows and convergence to global minimizers},
  author={Bah, Bubacarr and Rauhut, Holger and Terstiege, Ulrich and Westdickenberg, Michael},
  journal={Information and Inference: A Journal of the IMA},
  volume={11},
  number={1},
  pages={307--353},
  year={2022},
  publisher={Oxford University Press}
}

@article{menon2025entropy,
  title={An entropy formula for the Deep Linear Network},
  author={Menon, Govind and Yu, Tianmin},
  journal={arXiv preprint arXiv:2509.09088},
  year={2025}
}

@article{menon2025rle,
  title={A {R}iemannian {L}angevin equation for the deep linear network},
  author={Menon, Govind and Yu, Tianmin},
  journal={arXiv},
  pages={forthcoming},
  year={2025},
}

@inproceedings{menon2024geometry,
  title={The Geometry of the deep linear network},
  author={Menon, Govind},
  booktitle={XIV Symposium on Probability and Stochastic Processes: CIMAT, Mexico, November 20-24, 2023},
  pages={1},
  year={2025},
  organization={Springer Nature}
}

@article{cohen2023deep,
  title={Deep linear networks for matrix completion—an infinite depth limit},
  author={Cohen, Nadav and Menon, Govind and Veraszto, Zsolt},
  journal={SIAM Journal on Applied Dynamical Systems},
  volume={22},
  number={4},
  pages={3208--3232},
  year={2023},
  publisher={SIAM}
}

@book{kato2013perturbation,
  title={Perturbation theory for linear operators},
  author={Kato, Tosio},
  volume={132},
  year={2013},
  publisher={Springer Science \& Business Media}
}

@inproceedings{karmarkar1984new,
  title={A new polynomial-time algorithm for linear programming},
  author={Karmarkar, Narendra},
  booktitle={Proceedings of the sixteenth annual ACM symposium on Theory of computing},
  pages={302--311},
  year={1984}
}

@article{karmarkar1990riemannian,
  title={Riemannian geometry underlying interior-point methods for linear programming},
  author={Karmarkar, Narendra},
  journal={Contemp. Math.},
  volume={114},
  pages={51--75},
  year={1990},
  publisher={Amer. Math. Soc.}
}

@article{huang2023motion,
  title={Motion by mean curvature and {D}yson {B}rownian Motion},
  author={Huang, Ching-Peng and Inauen, Dominik and Menon, Govind},
  journal={Electronic Communications in Probability},
  volume={28},
  pages={1--10},
  year={2023},
  publisher={The Institute of Mathematical Statistics and the Bernoulli Society}
}

@book{olver2010nist,
  title     = {NIST Handbook of Mathematical Functions},
  editor    = {Olver, Frank W. J. and Lozier, Daniel W. and Boisvert, Ronald F. and Clark, Charles W.},
  publisher = {Cambridge University Press},
  year      = {2010},
  address   = {Cambridge},
}

\end{document}